\documentclass[journal,transmag]{IEEEtran}                % onecolumn (standard format)

\usepackage{mathptmx}
\usepackage{times}
\usepackage{helvet}
\usepackage{courier}
\usepackage{balance}  % to better equalize the last page
\usepackage{graphics} % for EPS, load graphicx instead
\usepackage{graphicx}
\usepackage[scriptsize]{subfigure}
\usepackage{algpseudocode}% http://ctan.org/pkg/algorithmicx
\usepackage{algorithm}% http://ctan.org/pkg/algorithm
\usepackage{epstopdf}
\usepackage{amsmath}
\usepackage{amssymb}
\usepackage{booktabs}
\usepackage{times}    % comment if you want LaTeX's default font
\usepackage{url}      % llt: nicely formatted URLs
\usepackage[english]{babel}
\usepackage[normalem]{ulem} % for the command \sout
\usepackage{soul} % % for the command \st
\usepackage{xcolor}
\usepackage{threeparttable}
\usepackage{multirow}
\usepackage{comment}
\usepackage{lipsum}
\usepackage[all,]{xy}

\usepackage{amsthm}

\renewcommand{\algorithmicrequire}{ \textbf{Input:}} %Use Input in the format of Algorithm
\renewcommand{\algorithmicensure}{ \textbf{Output:}} %UseOutput in the format of Algorithm
\newtheorem{theorem}{Theorem}
\newtheorem{lemma}[theorem]{Lemma}
\theoremstyle{definition}
\newtheorem{definition}{Definition}

\graphicspath{{./Figures/}}

\makeatletter
\def\url@leostyle{%
  \@ifundefined{selectfont}{\def\UrlFont{\sf}}{\def\UrlFont{\small\bf\ttfamily}}}
\makeatother
\makeatletter
\newif\if@borderstar
\def\bordermatrix{\@ifnextchar*{%
\@borderstartrue\@bordermatrix@i}{\@borderstarfalse\@bordermatrix@i*}%
}
\def\@bordermatrix@i*{\@ifnextchar[{\@bordermatrix@ii}{\@bordermatrix@ii[()]}}
\def\@bordermatrix@ii[#1]#2{%
\begingroup
\m@th\@tempdima8.75\p@\setbox\z@\vbox{%
\def\cr{\crcr\noalign{\kern 2\p@\global\let\cr\endline }}%
\ialign {$##$\hfil\kern 2\p@\kern\@tempdima & \thinspace %
\hfil $##$\hfil && \quad\hfil $##$\hfil\crcr\omit\strut %
\hfil\crcr\noalign{\kern -\baselineskip}#2\crcr\omit %
\strut\cr}}%
\setbox\tw@\vbox{\unvcopy\z@\global\setbox\@ne\lastbox}%
\setbox\tw@\hbox{\unhbox\@ne\unskip\global\setbox\@ne\lastbox}%
\setbox\tw@\hbox{%
$\kern\wd\@ne\kern -\@tempdima\left\@firstoftwo#1%
\if@borderstar\kern2pt\else\kern -\wd\@ne\fi%
\global\setbox\@ne\vbox{\box\@ne\if@borderstar\else\kern 2\p@\fi}%
\vcenter{\if@borderstar\else\kern -\ht\@ne\fi%
\unvbox\z@\kern-\if@borderstar2\fi\baselineskip}%
\if@borderstar\kern-2\@tempdima\kern2\p@\else\,\fi\right\@secondoftwo#1 $%
}\null \;\vbox{\kern\ht\@ne\box\tw@}%
\endgroup
}
\makeatother

\usepackage[pdftex]{hyperref}
\hypersetup{
pdftitle={},
pdfauthor={LaTeX},
pdfkeywords={},
bookmarksnumbered,
pdfstartview={FitH},
colorlinks,
citecolor=black,
filecolor=black,
linkcolor=black,
urlcolor=black,
breaklinks=true,
}

\begin{document}

\title{An Interval-Based Bayesian Generative Model for Human Complex Activity Recognition}

\author{\IEEEauthorblockN{Li Liu\IEEEauthorrefmark{1,2,3},~\IEEEmembership{Member,~IEEE}
Yongzhong Yang\IEEEauthorrefmark{4},
Lakshmi N. Govindarajan\IEEEauthorrefmark{4},
Shu Wang\IEEEauthorrefmark{5},\\
Bin Hu\IEEEauthorrefmark{6},~\IEEEmembership{Senior Member,~IEEE}
Li Cheng\IEEEauthorrefmark{3,4},~\IEEEmembership{Senior Member,~IEEE}, and
David S. Rosenblum\IEEEauthorrefmark{3},~\IEEEmembership{Fellow,~IEEE}
}
\IEEEauthorblockA{\IEEEauthorrefmark{1}Ministry of Education, Key Laboratory of Dependable Service Computing in Cyber Physical Society, Chongqing 400044, China}
\IEEEauthorblockA{\IEEEauthorrefmark{2}School of Software Engineering, Chongqing University, Chongqing 400044, China}
\IEEEauthorblockA{\IEEEauthorrefmark{3}School of Computing, National University of Singapore, 117417, Singapore}
\IEEEauthorblockA{\IEEEauthorrefmark{4}Bioinformatics Institute, A*STAR, 138671, Singapore}
\IEEEauthorblockA{\IEEEauthorrefmark{5}College of Material Science and Engineering, Lanzhou University of Technology, Lanzhou 730050, China}
\IEEEauthorblockA{\IEEEauthorrefmark{6}School of Information Science and Engineering, Lanzhou University, 730000, China}}
%\thanks{Manuscript received XXX; revised XXX. Co-corresponding authors: Li Liu (email: dcsliuli@cqu.edu.cn) and Bin Hu (email: bh@lzu.edu.cn).}}

%\markboth{IEEE Transactions on Cybernetics,~Vol.~XX, No.~X, XXX}%
%{Shell \MakeLowercase{\textit{Liu L. et al.}}: An Interval-Based Bayesian Generative Model for Human Complex Activity Recognition}

\IEEEtitleabstractindextext{%
\begin{abstract}
Complex activity recognition is challenging due to the inherent uncertainty and diversity of performing a complex activity.
Normally, each instance of a complex activity has its own configuration of atomic actions and their temporal dependencies.
We propose in this paper an atomic action-based Bayesian model that constructs Allen's interval relation networks to characterize complex activities with structural varieties in a probabilistic generative way:
By introducing latent variables from the Chinese restaurant process, our approach is able to capture all possible styles of a particular complex activity as a unique set of distributions over atomic actions and relations.
We also show that local temporal dependencies can be retained and are globally consistent in the resulting interval network.
Moreover, network structure can be learned from empirical data.
A new dataset of complex hand activities has been constructed and made publicly available,
which is %about an order of magnitude
much larger in size than any existing datasets.
Empirical evaluations on benchmark datasets as well as our in-house dataset demonstrate the competitiveness of our approach.
%In particular, it is shown that our model is rather robust to the errors introduced by the low-level atomic action predictions from raw signals.
\end{abstract}

}

\maketitle

\IEEEdisplaynontitleabstractindextext

% \fi
\IEEEpeerreviewmaketitle

\section{Introduction}

A \emph{complex activity} consists of a set of temporally-composed events of \emph{atomic actions}, which are the lowest-level events that can be directly detected from sensors. In other words, a complex activity is usually composed of multiple atomic actions occurring consecutively and concurrently over a duration of time.
Modeling and recognizing complex activities remains an open research question as it faces several challenges:
First, understanding complex activities calls for not only the inference of atomic actions, but also the interpretation of their rich temporal dependencies.
Second, individuals often possess diverse styles of performing the same complex activity. As a result, a complex activity recognition model should be capable of capturing and propagating the underlying uncertainties over atomic actions and their temporal relationships.
Third, a complex activity recognition model should also tolerate errors introduced from atomic action level, due to sensor noise or low-level prediction errors.

\subsection{Related Work}
%\noindent\textbf{\emph{Related Work}}
%\paragraph{Related Work.}
%We attempt to summarize in this section related work on complex activity recognition.
Currently, a lot of research focuses on semantic-based complex activity modeling.
%Semantic-based models are capable of representing rich temporal relations, but they often do not have expressive power to capture uncertainties.
Many semantic-based models such as context-free grammar (CFG)~\cite{ryoo2009semantic} and Markov logic network (MLN)~\cite{helaoui2011recognizing,liu2016mining}) are used to represent complex activities, which can handle rich temporal relations. Yet formulae and their weights in these models (e.g. CFG grammars and MLN structures) need to be manually encoded, which could be rather difficult to scale up and is almost impossible for many practical scenarios where temporal relations among activities are intricate.
Although a number of semantic-based approaches have been proposed for learning temporal relations, such as stochastic context-free grammars~\cite{Veeraraghavan2007Learning} and Inductive Logic Programming (ILP)~\cite{dubba2015learning}, they can only learn formulas that are either true or false, but cannot learn their weights, which hinders them from handling uncertainty.

On the other hand, graphical models become increasingly popular for modeling complex activities because of their capability of managing uncertainties~\cite{zhang2013modeling}.
Unfortunately, most of them can handle three temporal relations only, i.e. equals, follows and precedes.
Both Hidden Markov model (HMM) and conditional random field (CRF) are commonly used for recognizing sequential activities, but are limited in managing overlapping activities~\cite{kim2010human}.
Many variants with complex structures have been proposed to capture more temporal relations among activities, such as interleaved hidden Markov models (IHMM)~\cite{modayil2008improving}, skip-chain CRF~\cite{hu2008cigar} and so on.
However, these models are time point-based, and hence with the growth of the number of concurrent activities they are highly computationally intensive~\cite{pinhanez1999representation}.
Dynamic Bayesian network (DBN) can learn more temporal dependencies than HMM and CRF by adding activities' duration states, but imposes more computational burden~\cite{oliver2005comparison}.
Moreover, the structures of these graphical models are usually manually specified instead of learned from the data.
The interval temporal Bayesian network (ITBN)~\cite{zhang2013modeling} differs significantly from the previous methods, as being a graphical model that first integrates interval-based Bayesian network with the 13 Allen's relations. %It utilizes the Allen's 13 relations with its Bayesian network.
Nonetheless, ITBN has several significant drawbacks:
First, its directed acyclic Bayesian structure makes it have to ignore some temporal relations to ensure a temporally consistent network. As such, it may result in loss of internal relations.
Second, it would be rather computationally expensive to evaluate all possible consistent network structures, especially when the network size is large.
Third, neither can ITBN manage the multiple occurrences of the same atomic action, nor can it handle arbitrary network size as it remains unchanged as the count of atomic action types.
Figure~\ref{fig:comparison-structure-methods} illustrates the graph structures of the three commonly-used graphical models.
\begin{figure}[!htbp]
\center
\subfigure[IHMM]{
\includegraphics[width=0.27\columnwidth]{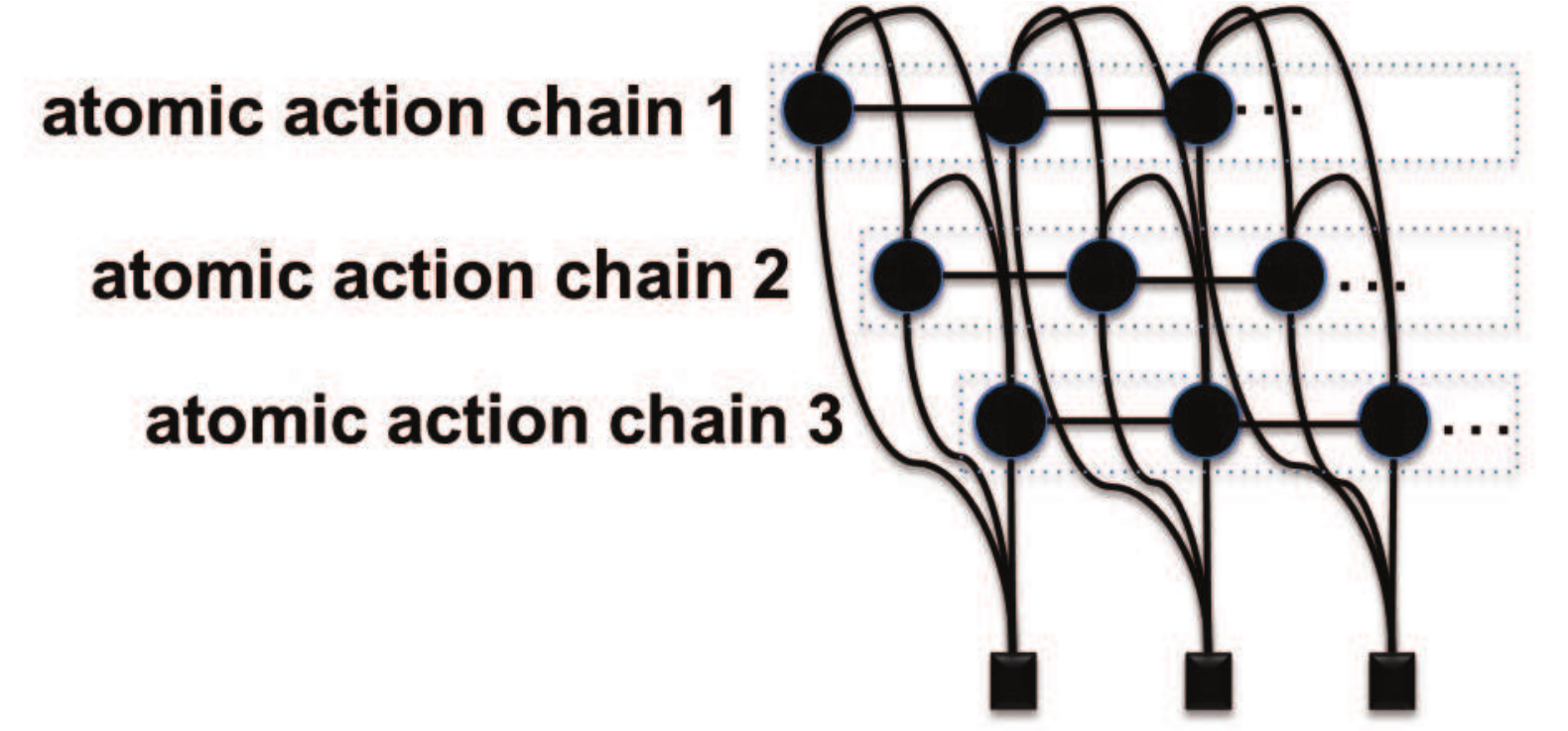}
\label{fig:experiment-comparison-ihmm}
}
%\subfigure[SCCRF]{
%\includegraphics[width=0.28\columnwidth]{Fig-experiment-comparison-sccrf.eps}
%\label{fig:experiment-comparison-sccrf}
%}
\subfigure[DBN]{
\includegraphics[width=0.32\columnwidth]{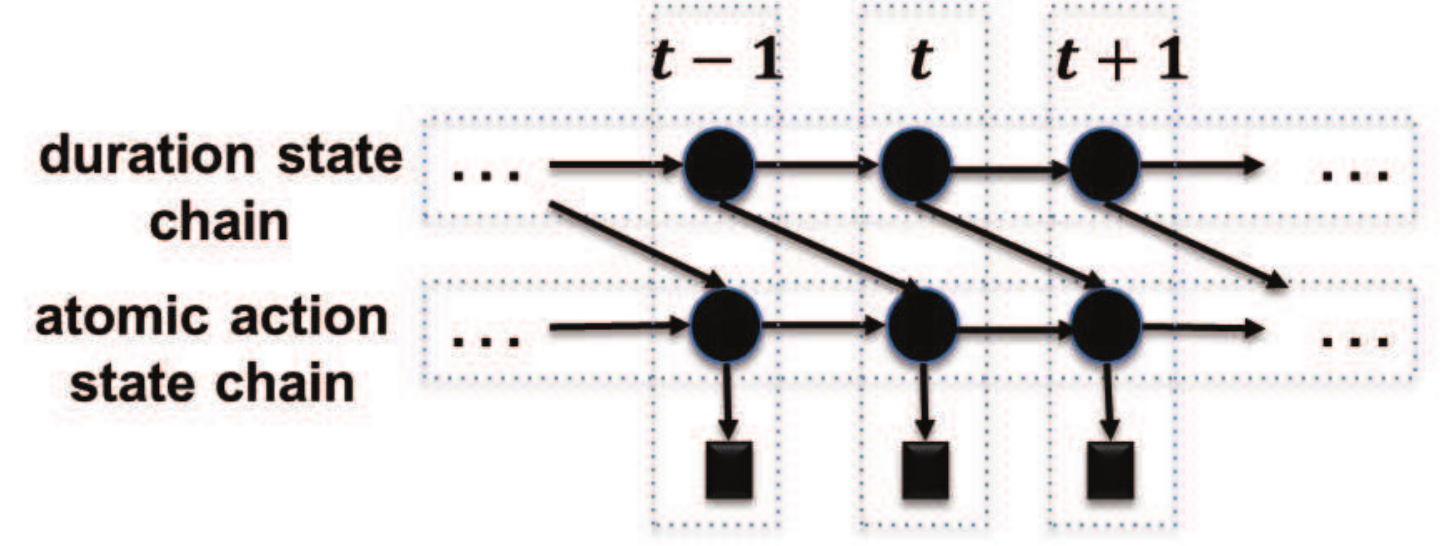}
\label{fig:experiment-comparison-dbn}
}
\subfigure[ITBN]{
\includegraphics[width=0.3\columnwidth]{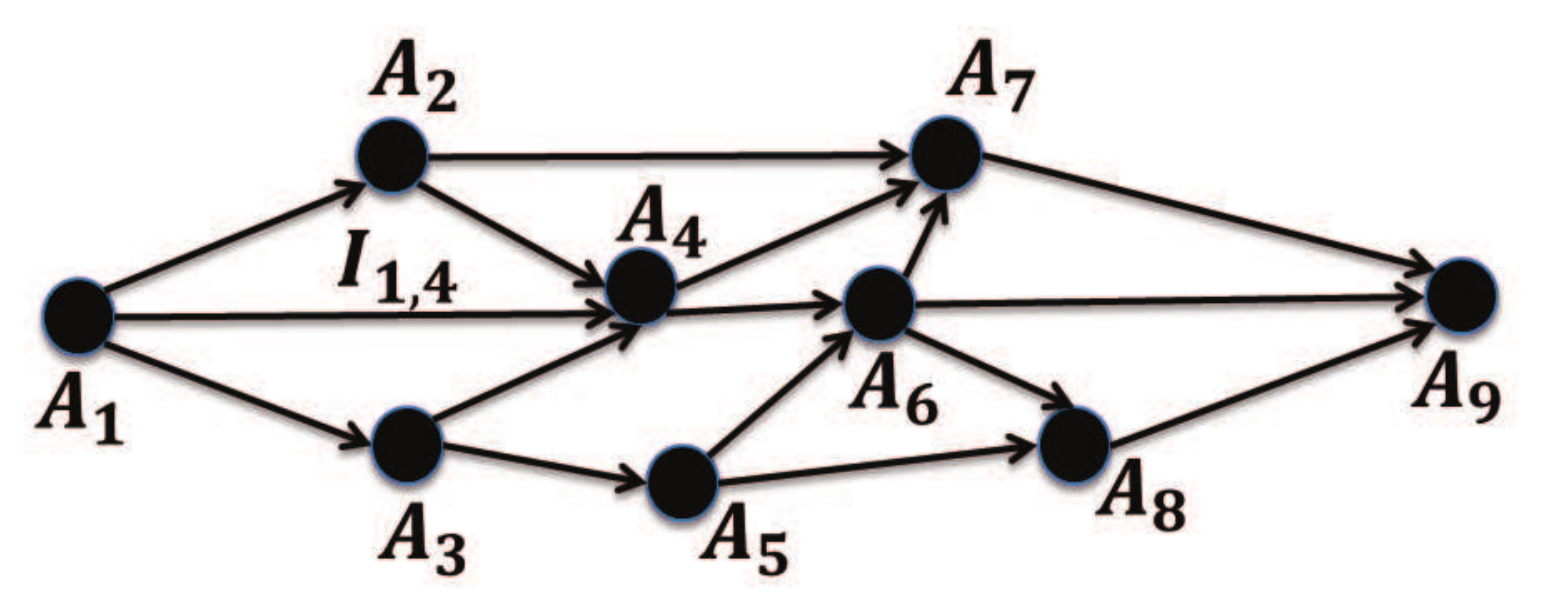}
\label{fig:experiment-comparison-itbn}
}
\caption{The structures of three graphical models for complex activity recognition. (a)IHMM, where the observations of atomic actions (square-shaped nodes) and several chains of hidden states (round-shaped nodes) are used to handle overlapping; (b) DBN, where duration states and atomic action states are represented as chains of nodes; (c) ITBN, where any atomic action type ($A_1$-$A_9$) is represented by a node and the set of all possible interval relations between any pair of atomic action types $A_i$ and $A_j$ is represented by a link $I_{i,j}$.}
\label{fig:comparison-structure-methods}
\end{figure}
It is worth noting that we will focus on complex activity recognition in this paper, and interested readers may consult the excellent reviews~\cite{aggarwal2011human,chen2012sensor,lara2013survey,cook2013activity,bulling2014tutorial,devanne20143} for further details regarding atomic-level action recognition.
% our contributions
\subsection{Our approach}
To address the problems in the previous models,
%To address these challenges in complex activity recognition,
we present an \emph{interval-based Bayesian generative network} (IBGN) to explicitly model complex activities with inherent structural varieties, which is achieved by constructing probabilistic interval-based networks with temporal dependencies.
In other words, our model considers a probabilistic generative process of constructing interval-based networks to characterize the complex activities of interests.
Briefly speaking, a set of latent variables, called \emph{tables}, which are generated from the \emph{Chinese restaurant process} (\emph{CRP})~\cite{pitman2002combinatorial} are introduced to construct the interval-based network structures of a complex activity.
Each latent variable characterizes a unique \emph{style} of this complex activity by containing its distinct set of atomic actions and their temporal dependencies based on Allen's interval relations.
There are two advantages to using \emph{CRP}: It allows our model to describe a complex activity of arbitrary interval sizes and also to take into account multiple occurrences of the same atomic actions.
We further introduce \emph{interval relation constraints} that can guarantee the whole network generation process is globally temporally consistent without loss of internal relations.
In addition, instead of manually specify a network to a fixed structure, the network structure in our approach is learned from training data.
By \emph{learning network structures}, our model is more effective than existing graphical models in characterizing the inherent structural variability in complex activities. A further comparison study is summarized in Table~\ref{tab-comparison-graphical-model}, which also shows our main contributions.
\begin{table}[!htbp]
\caption{A summary comparison of graphical model-based approaches for recognizing complex activities.}
\scriptsize
\centering
\scalebox{0.9}{
\begin{tabular}{|p{0.63\columnwidth}|p{0.095\columnwidth}|p{0.07\columnwidth}|p{0.07\columnwidth}|}
\hline
& HMMs \& BNs   &  ITBN     & our IBGN\\
\hline
Time-point-based (p) or interval-based (i)? & p  & i & i\\
How many temporal relations can be described? & 3  & 13 & 13\\
Retain all the interval relations during training stage?& \checkmark  & \text{\sffamily X} & \checkmark\\
Handle any possible combinations of interval relations? & \text{\sffamily X}  & \checkmark & \checkmark\\
Handle multiple occurrences of the same atomic action? & \checkmark  & \text{\sffamily X} & \checkmark\\
Handle variable number of overlapping actions? & \text{\sffamily X} & \checkmark & \checkmark\\
Describe a complex activity with variable sizes of points or intervals?& \checkmark & \text{\sffamily X} & \checkmark\\
Is the structure learned from training data? & \text{\sffamily X}  & \checkmark & \checkmark\\
%Can run in polynomial computational time? & \text{\sffamily X} & \text{\sffamily X} & \text{\sffamily X} & \checkmark\\
\hline
\end{tabular}
}
\label{tab-comparison-graphical-model}
\end{table}

It is worth mentioning that in spite of the increasing need from diverse applications in the area of complex activity recognition, there are only a few publicly-available complex activity recognition datasets~\cite{roggen2010collecting,brendel2011probabilistic,lillo2014discriminative}.
In particular, the number of instances are on the order of hundreds at most.
This motivates us to propose a dedicated large-scale dataset on depth camera-based complex hand activity recognition. We have constructed a new dataset of complex hand activities which contains instances that are about an order of magnitude larger than the existing datasets.
We have made the dataset and related tools made publicly available on a dedicated project website\footnote{The dataset including raw videos, annotations and related tools can be found at \url{http://web.bii.a-star.edu.sg/~lakshming/complexAct/}.} in support of the open-source research activities in this emerging research community.

\section{Definitions and Problem Formulation}
Assume we have at hand a dataset $\mathbb{D}$ of $N$ instances from a set of $L$ types of complex activities involving a set of $M$ types of atomic actions $\mathbb{A}=\{A_1, A_2, \ldots, A_M\}$.
An \emph{atomic action interval} (\emph{interval} for short) is written by $I_i=a_i@[t_i^-,t_i^+)$, where $t_i^-$ and $t_i^+$ represent the start and end time of the atomic action $a_i\in\mathbb{A}$, respectively.
Each complex activity instance is a sequence of $k$ ordered intervals, i.e. $\langle I_1,I_2,\ldots,I_k \rangle$, such that if $1\leq i<j\leq k$, then $(t_i^-<t_j^-)$ or $(t_i^-=t_j^-\text{ and }t_i^+\leq t_j^+)$.
Seven Allen's interval relations (relations for short)~\cite{allen1983maintaining} are used to represent all possible temporal relationships between two intervals, denoted by $\mathbb{R}=\{b,m,o,s,c,f,\equiv\}$, which is summarized below:
\begin{equation}
\nonumber
\footnotesize
\begin{split}
(\texttt{before})\quad I_i~ b~ I_j \quad\text{:}\quad& t_i^+<t_j^-,\\
(\texttt{meet})\quad I_i~ m~ I_j \quad\text{:}\quad& t_i^+=t_j^-, \\
(\texttt{overlap})\quad I_i~ o~ I_j \quad\text{:}\quad& t_i^-<t_j^-<t_i^+<t_j^+, \\
(\texttt{start})\quad I_i~ s~ I_j \quad\text{:}\quad& t_i^-=t_j^-\text{ and }t_i^+<t_j^+, \\
(\texttt{contain})\quad I_i~ c~ I_j \quad\text{:}\quad& t_i^-<t_j^-\text{ and }t_j^+<t_i^+, \\
(\texttt{finished-by})\quad I_i~ f~ I_j \quad\text{:}\quad& t_i^-<t_j^-\text{ and }t_i^+=t_j^+, \\
(\texttt{equal})\quad I_i~ \equiv~ I_j\quad\text{:}\quad& t_i^-=t_j^-\text{ and }t_i^+=t_j^+.
\end{split}
\end{equation}

We define an interval-based network (network for short) $G=(V,E)$ to represent a complex activity containing the temporal relationships between intervals. A node $v_i\in V$ represents the $i$-th interval (i.e. $I_i$) in a complex activity instance, while a directed link $e_{i,j}\in E$ represents the relation between two intervals (i.e. $I_i$ and $I_j$), where $1\leq i<j\leq \lvert V\rvert$ ($\lvert V\rvert$ is the cardinality of the set $V$). Note a link always starts from a node with a smaller index than its arrival node.
Each link $e_{i,j}$ involves one and only one relation $r_{i,j}\in\mathbb{R}$. Figure~\ref{fig:example} illustrates the corresponding networks of a complex activity.
\begin{figure}
\centering
\subfigure[Offensive play I]{
\includegraphics[width=0.2\textwidth]{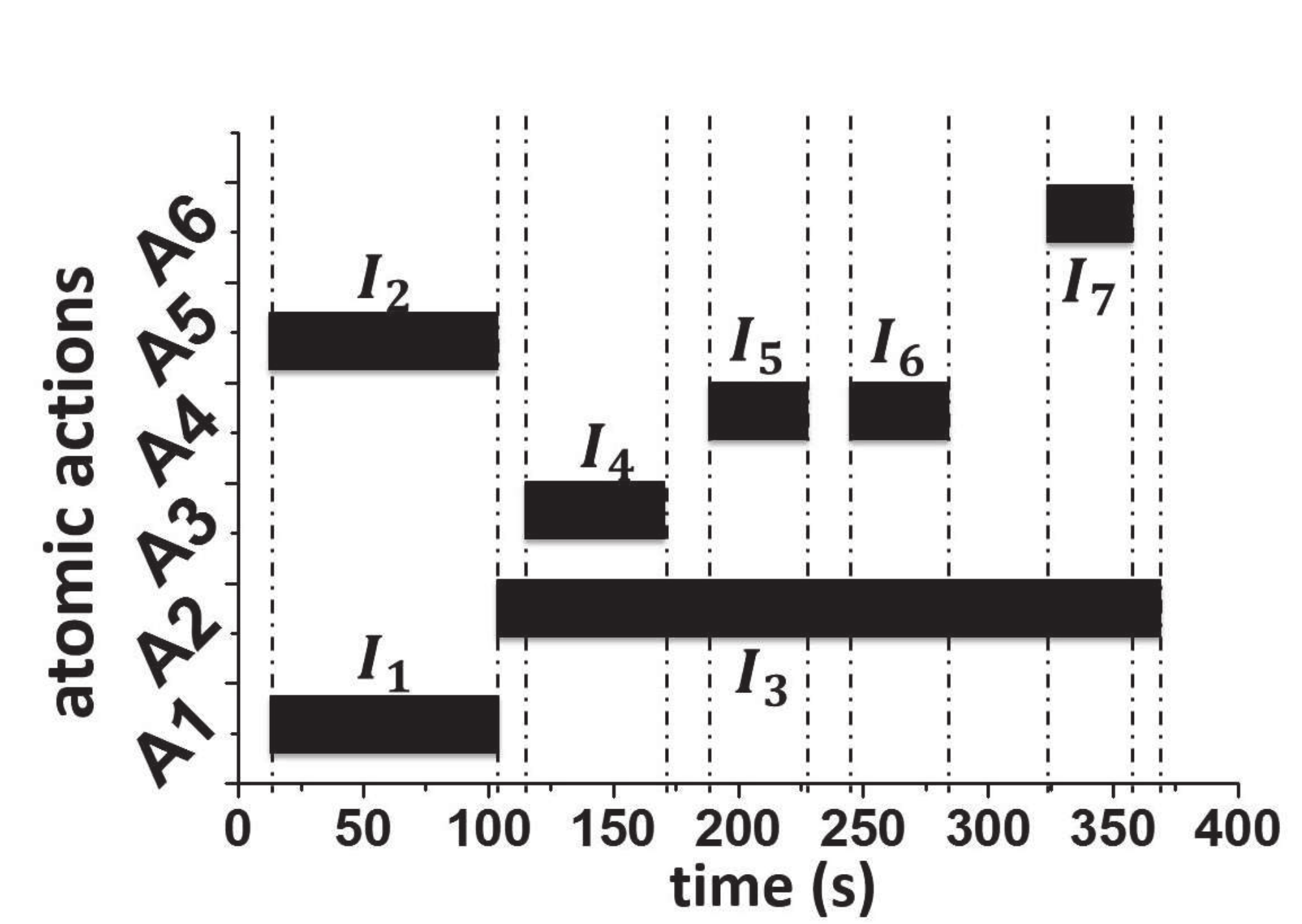}
\label{fig:example-timeline-case1}
 }
\subfigure[Offensive play II]{
\includegraphics[width=0.2\textwidth]{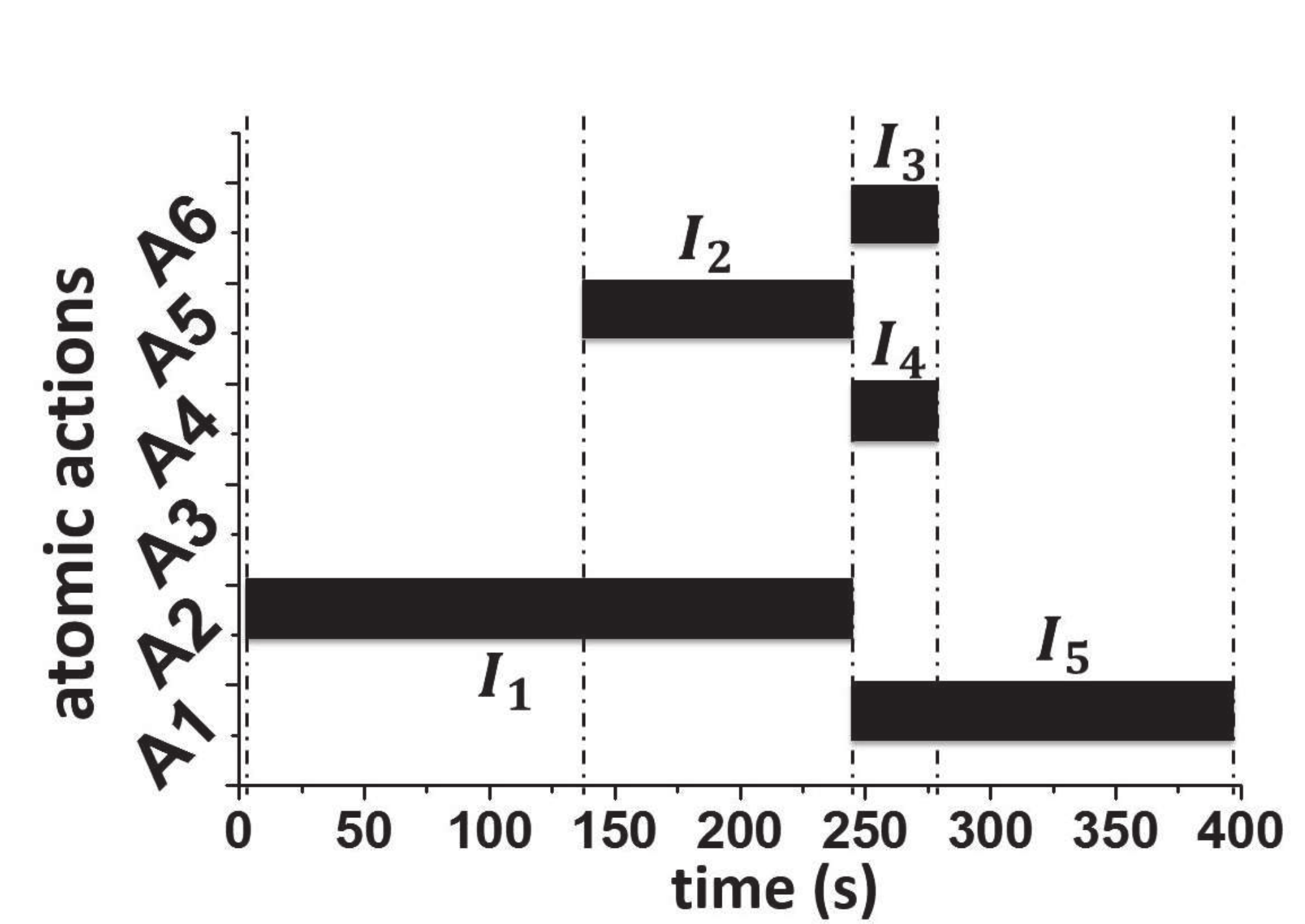}
\label{fig:example-timeline-case2}
 }
\subfigure[Network I]{
\includegraphics[width=0.2\textwidth]{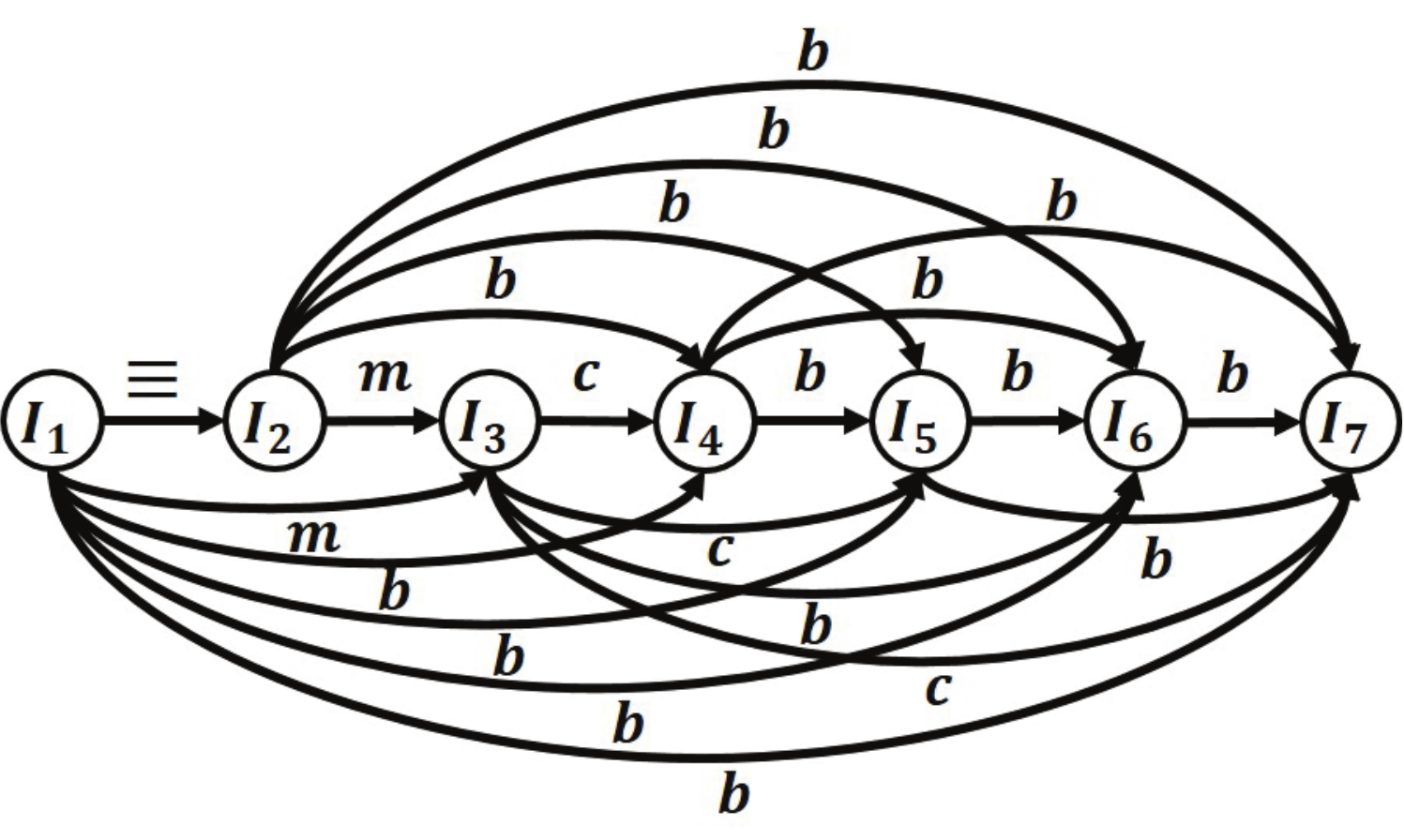}
\label{fig:example-network-case1}
 }
\subfigure[Network II]{
\includegraphics[width=0.2\textwidth]{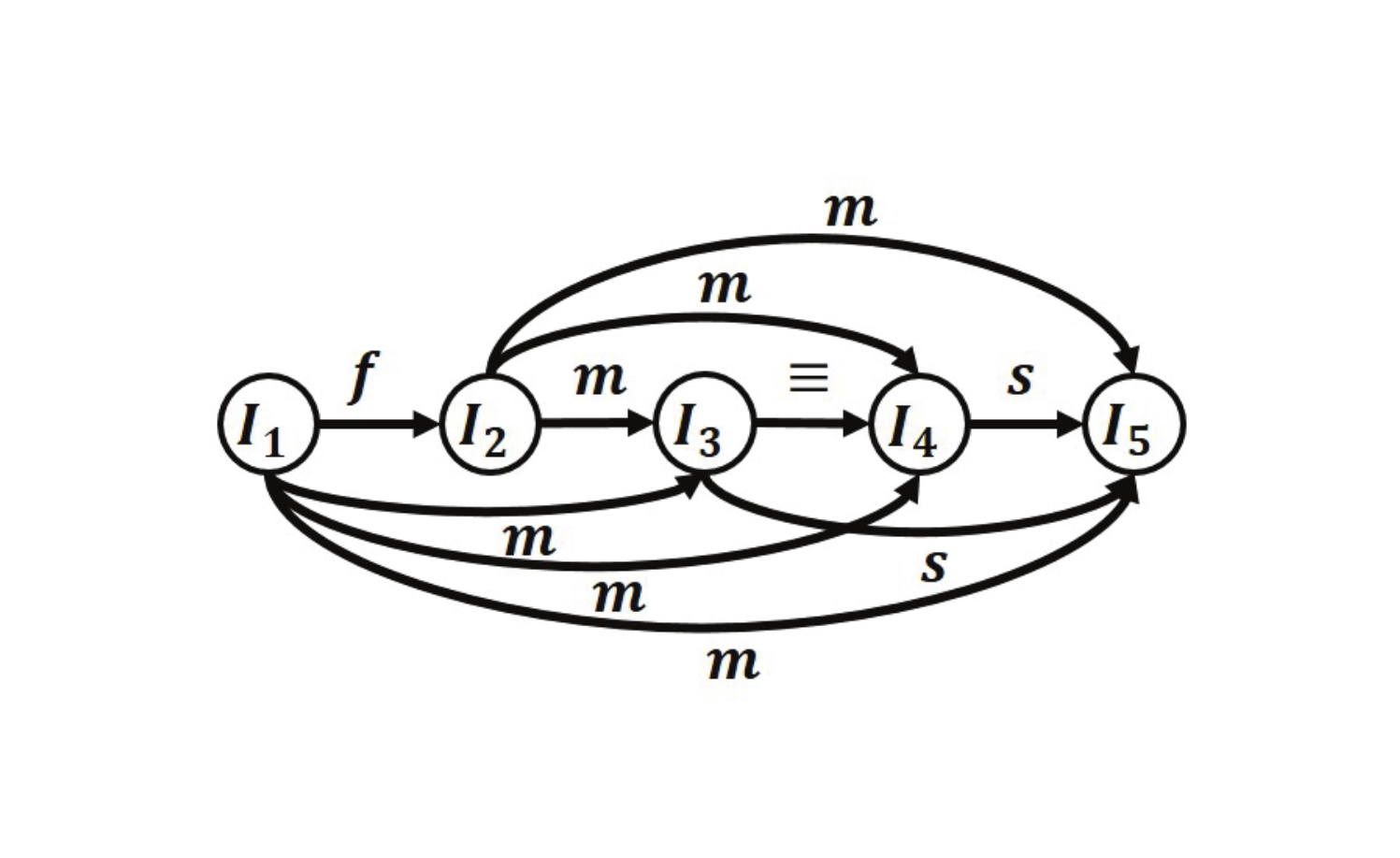}
\label{fig:example-network-case2}
 }
\caption{Two possible styles of the complex activity \emph{offensive play} and its corresponding networks where the nodes represent the intervals of atomic actions and the links represent their temporal relations. Atomic actions: $A_1=$\emph{walk}, $A_2=$\emph{stand}, $A_3=$\emph{hold ball}, $A_4=$\emph{jump}, $A_5=$\emph{dribble}, $A_6=$\emph{shoot}.}
\label{fig:example}
\end{figure}

The temporal relations on links shall be globally consistent in any interval-based network.
Given any two relations $r_{i,j},r_{j,k}\in\mathbb{R}$ on links $e_{i,j}$ and $e_{j,k}$, respectively, the relation $r_{i,k}$ on link $e_{i,k}$ must follow the transitivity properties as shown in Figure~\ref{fig:transitivity}.
For example, suppose $I_i$ \texttt{meets} $I_j$ and $I_j$ \texttt{starts} $I_k$, $I_i$ is certainly \texttt{meets} $I_k$. However, if $I_i$ \texttt{starts} $I_j$ and $I_j$ \texttt{is finished by} $I_k$, there are three possible relations between $I_i$ and $I_k$, that is, \texttt{before}, \texttt{meet} and \texttt{overlap}.
We formally use $r_{i,j}\circ r_{j,k}$ to denote such \emph{composition} operation following the transitivity properties.
%It can be seem that for any $r_{i,j}, r_{j,k}\in\mathbb{R}$, $r_{i,j}\circ r_{j,k}\subseteq\mathbb{R}$.
We say a network is \emph{consistent} such that $r_{i,k}\in r_{i,j}\circ r_{j,k}$ for any triplet of links $e_{i,j}$, $e_{j,k}$ and $e_{i,k}$ in the network.
%Notice that the interval-based network is different from the interval algebra network~\cite{allen1983maintaining} where each link are labeled with the union of all possible interval relations. In this paper, the term \emph{network} refers to the interval-based network in our definition, and \emph{relation} refers to the interval relation in $\mathbb{R}$.
\begin{figure}[!htbp]
\center
\includegraphics[width=0.75\columnwidth]{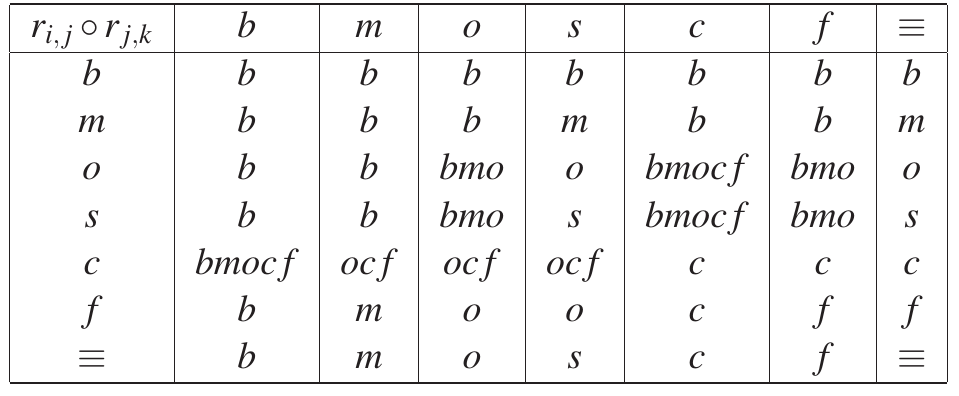}
\caption{The transitivity table for any interval relation $r_{i,k}$ through composition operation $r_{i,j}\circ r_{j,k}$.}
\label{fig:transitivity}
\end{figure}

It is clear that a network can characterize one possible style of a complex activity with diverse combinations of atomic actions and their interval relations, as illustrated in Figure~\ref{fig:example}.
From another point of view, a complex activity can be instantiated by sampling atomic actions and relations assigned to their associated nodes and links in a network following certain probabilities.
In this way, a recognition model built on such networks is able to handle uncertainty in complex activity recognition.
In addition, multiple occurrences of the same atomic action can appear in the same network but in different nodes (i.e. intervals).
To this end, we present the probabilistic generative model IBGN where these networks can be constructed following the styles of the complex activities of interests under uncertainty.
We shall also consider the IBGN model to construct networks with variable sizes of nodes and arbitrary structures.
Note in our approach, for each type of complex activities we would learn one such dedicated IBGN model.

\section{Our IBGN Model}
For any complex activity type $l$ ($1\leq l\leq L$), denote $\mathbb{D}_l\subseteq \mathbb{D}$ the corresponding subset of $N_l$ instances, where each element $d\in\mathbb{D}_l$ is an instance of the $l$-th type of complex activity
%A generative model is a model for randomly generating observable data values, typically given some hidden parameters.
In IBGN, the generative process of constructing an interval-based network $G_d=(V_d,E_d)$ for describing the observed instance $d$ consists of two parts, node generation and link generation, which are described below.

\subsection{Node Generation}
In our IBGN model, we consider generating a network where each node is associated with an atomic action in a probabilistic manner. We also require our model to be capable of generating variable numbers of nodes and handling multiple occurrences of the same atomic action in our network, as summarized in Table~\ref{tab-comparison-graphical-model}.

\begin{figure*}[!htbp]
\centering
\includegraphics[width=0.65\textwidth]{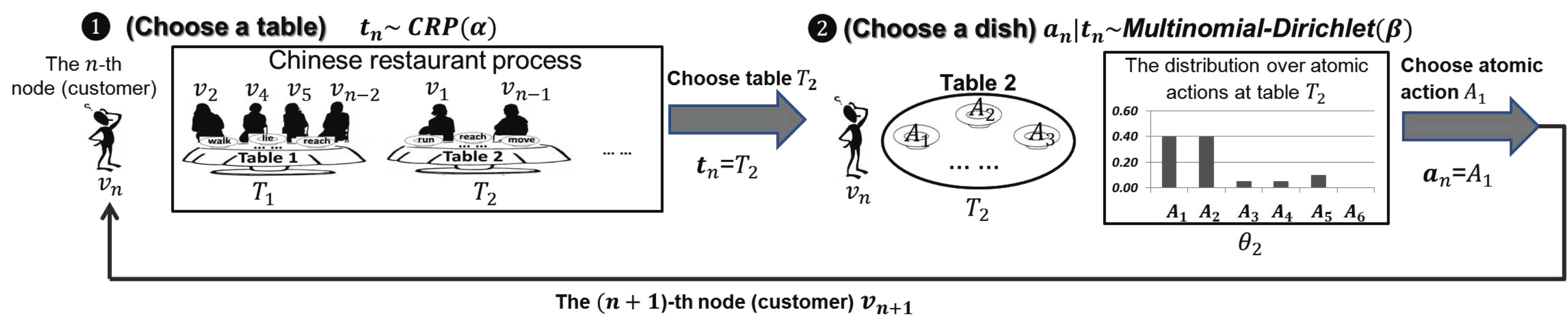}
\caption{An illustration of the generative process of choosing a table $\mathbf{t}_n$ and an atomic action $\mathbf{a}_n$ for node $v_n$.}
\label{fig:example-node-generation}
\end{figure*}

To accomplish these tasks, we first extend the process of the \emph{CRP}s to make it available in our IBGN model. Originally, a \emph{CRP} is analogous to a stochastic process of choosing tables for customers in a Chinese restaurant.
In a nutshell, suppose there are an infinite number of tables $\{T_1,T_2,\ldots\}$. The first customer ($n=1$) always selects the first table; Any other customer ($n>1$) randomly selects an occupied table or an unoccupied empty table with a certain probability.

We continue this process by serving dishes right after each customer is seated at a table. Assume there are a finite number of dishes and an infinite number of cuisines. Each table is associated with a cuisine that dishes are served with a unique probability distribution. Any customer sitting at a table randomly selects a dish with the probability relating to its corresponding cuisine.

In our model, a network contains a group of customers where each customer is analogous to a node while a dish is analogous to an atomic action type. Suppose customers from the same group prefer similar cuisines, which is analogous to a complex activity of interest, and are more likely to sit at the same tables.
Formally, denote $\{\mathbf{t}_1,\mathbf{t}_2,\ldots\}$ the variables of tables, and $\{\mathbf{a}_1,\mathbf{a}_2,\ldots\}$ the variables of atomic actions~\footnote{Whenever possible, we would use bold lowercase letters such as $\mathbf{t}_n$, $\mathbf{a}_n$ to represent variables, and uppercase letters such as $T_i$ and $A_j$ to represent generic values of these variables.}.
To construct a network $G_d=(V_d,E_d)$, $\mathbf{t}_n$ and $\mathbf{a}_n$ are the table and the atomic action (dish) assigned to the node (customer) $v_n\in V_d$ ($n=1,2,\ldots,\lvert V_d\rvert$).
The generative process operates as follows. The $n$-th node $v_{n}$ selects a table $\mathbf{t}_n$ that is drawn from the following distribution:
\begin{equation}
\small
P(\mathbf{t}_n=T_\zeta\mid\mathbf{t}_1,\ldots,\mathbf{t}_{n-1})=\left\{
\begin{array}{l l}
\frac{nt_\zeta}{n+\alpha-1} & \quad \text{if $T_\zeta$ is a non-empty table },\\
\frac{\alpha}{n+\alpha-1} & \quad \text{if $T_\zeta$ is a new table},
\end{array} \right.
\end{equation}
where $nt_{\zeta}$ is the number of existing nodes occupied at table $T_\zeta$ ($\zeta=1,2,\ldots$), with $\sum_{\zeta}{nt_\zeta}=n-1$, $\mathbf{t}=T_1$. $\alpha$ is a tuning hyperparameter.
It is worth mentioning that the distribution over table assignments in \emph{CRP} is invariant and exchangeable under permutation according to de Finetti's Theorem~\cite{teh2011dirichlet}.
After $v_{n}$ is assigned with a table $\textbf{t}_n=T_\zeta$, an atomic action (dish) $\mathbf{a}_{n}$ is chosen from the table by:
\begin{equation}
\nonumber
\small
\begin{split}
&(\mathbf{a}_n\mid\mathbf{t}_n) \sim\text{\emph{Multinomial}}(\theta_{\zeta}),\\
&\theta_{\zeta}\sim\text{\emph{Dirichlet}($\beta$)},
\end{split}
\end{equation}
where $\beta$ is a hyperparameter.
%Note the Dirichlet processes extend \emph{CRP} by serving each table a set of different, independently chosen atomic action (dishes).
Note $\theta_\zeta=(\theta_{\zeta,1},\ldots,\theta_{\zeta,M})^T$ is the parameter vector of a multinomial distribution at table $T_\zeta$. Figure~\ref{fig:example-node-generation} presents a cartoon explanation of this node generation process of our IBGN model.
\begin{figure}[!htbp]
\centering
\subfigure[Distributions of atomic actions at different tables that collectively represent the complex activity \emph{offensive play}.]{
\includegraphics[width=0.9\columnwidth]{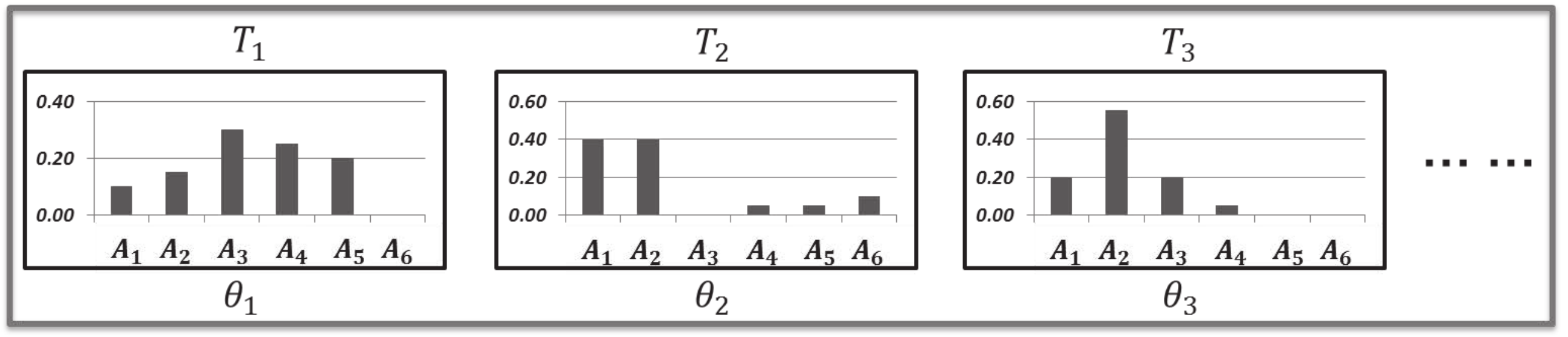}\label{fig:example-crp-table-case1}
}
\subfigure[Distributions of atomic actions at different tables that collectively represent the complex activity \emph{defensive play}.]{
\includegraphics[width=0.9\columnwidth]{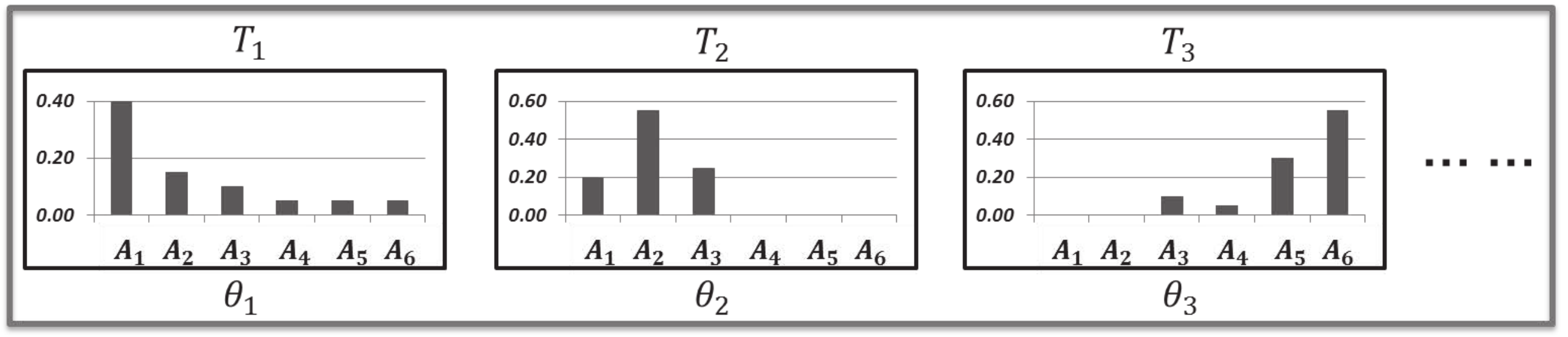}\label{fig:example-crp-table-case2}
}
\caption{Examples of the tables and corresponding distributions over atomic actions for representing two complex activities, \texttt{offensive play} and \texttt{defensive play}, respectively. Here each table contains a set of atomic actions under a specific distribution.}
\label{fig:example-tables}
\end{figure}
Since we would obtain one dedicated IBGN model for each complex activity, a complex activity is now characterized as a unique set of distributions over atomic actions, i.e. $\theta=\{\theta_1,\theta_2,\ldots\}$. As illustrated in Figure~\ref{fig:example-tables}, the two different complex activities \emph{offensive play} and \emph{defensive play} are associated with two distinct sets of tables with their own distributions over atomic actions. It suggests that our IBGN modeling approach is capable of differentiating the underlying characteristics associated with atomic actions from different complex activity categories.
%This is achieved by employing the one vs. all learning strategy where an IBGN model is separately learned for each activity category.
%Moreover, since every table in the restaurant can accommodate arbitrary number of customers, our model is able to describe a complex activity containing variable number of intervals (i.e. nodes) as well as multiple occurrences of the same atomic actions.

\subsection{Link Generation}
Once each node (interval) is assigned to an atomic action, links and their associated relations are to be generated next. It is important to ensure \emph{consistency} of the resulting relations over all links.
%In the previous model, relations between any pair of nodes are considered, which results in a fully linked network. An \emph{interval relation generation constraint} was introduced to ensure temporal consistency on such fully connected structures during the relation generation procedure. In this work, we relax the assumption that a network is fully connected and redefine the \emph{constraint} to handle temporal consistency for arbitrary network structures.
Formally, given two relations $\mathbf{r}_{n',n''}$ and $\mathbf{r}_{n'',n}$ on the links $e_{n',n''}$ and $e_{n'',n}$ ($n'<n''<n$), respectively, the interval relation $\mathbf{r}_{n',n}$ on link $e_{n',n}$ shall follow the transitivity properties listed in Figure~\ref{fig:transitivity}.
It is straightforward to verify that the set $\mathbb{R}$ is closed under the composition operation.
As a result, by applying the transitivity table, for any composition there exists only $11$ possible transitive relations in Figure~\ref{fig:union}, denoted as $\mathbb{C}=\{C_z:1\leq z\leq 11\}$. In other words, any composition of two consecutive relations satisfies $\mathbf{r}_{n',n''}\circ \mathbf{r}_{n'',n}\in\mathbb{C}$.
\begin{figure}[!htbp]
\centering
\includegraphics[width=\columnwidth]{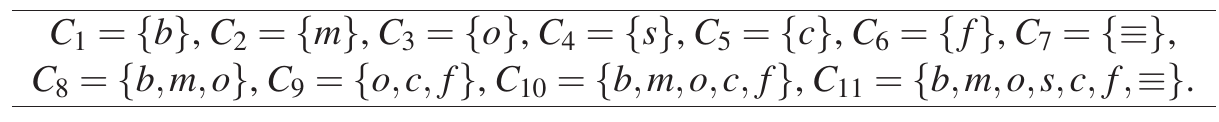}
\caption{The $11$ possible interval composition relations.}
\label{fig:union}
\end{figure}

To construct a \emph{globally consistent network}, the relations on every triangle in a network must also be consistent. Namely, for any triplet of nodes $v_{n'}$, $v_{n''}$ and $v_n$, if there exist three links $\mathbf{r}_{n',n}$, $\mathbf{r}_{n',n''}$, and $\mathbf{r}_{n'',n}$, they need to satisfy the transitive relation $\mathbf{r}_{n',n}\in \mathbf{r}_{n',n''}\circ \mathbf{r}_{n'',n}$.
As such, we define the \emph{interval relation constraint} variable as follows:
\begin{definition}
\begin{comment}
Give an arbitrary interval-based network $G=(V,E)$, the interval relation constraint $\mathbf{c}_{n',n}\in\mathbb{C}$ for a link $e_{n',n}\in E$ ($1\leq n'<n\leq \lvert V \rvert$) is
\begin{equation}
\label{eq:intervalconstraint}
\nonumber
\begin{aligned}
&\mathbf{c}_{n',n} = \left\{
  \begin{array}{l l}
    \bigcap_{n''=n'+1}^{n-1}{(\mathbf{r}_{n',n''}\circ\mathbf{r}_{n'',n})} & \quad \text{if $n>n'+1$},\\
    \mathbb{R} & \quad \text{if $n=n'+1$}.
  \end{array} \right.\\
\end{aligned}
\end{equation}
\end{comment}
%\begin{comment}
Give an arbitrary interval-based network $G=(V,E)$, the interval relation constraint $\mathbf{c}_{n',n}\in\mathbb{C}$ for a link $e_{n',n}\in E$ ($1\leq n'<n\leq \lvert V \rvert$) is
\begin{equation}
\label{eq:intervalconstraint}
\nonumber
\small
\begin{aligned}
&\mathbf{c}_{n',n} = \left\{
  \begin{array}{l l}
    \bigcap_{n''=n'+1}^{n-1}{(\mathbf{x}_{n',n''}\circ\mathbf{x}_{n'',n})} & \quad \text{if $n>n'+1$},\\
    \mathbb{R} & \quad \text{if $n=n'+1$},
  \end{array} \right.\\
&\text{where}\quad
\mathbf{x}_{p,q} = \left\{
  \begin{array}{l l}
    \mathbf{r}_{p,q} & \quad \text{if }e_{p,q}\in E,\\
    \mathbf{c}_{p,q} & \quad \text{if }e_{p,q}\notin E,
  \end{array} \right.
  (n'\leq p<q\leq n).
\end{aligned}
\end{equation}
%\end{comment}
\end{definition}
%Here, $\mathbf{X}\circ\mathbf{X}'=\mathop{\bigcup}\limits_{\mathbf{r}\in\mathbf{X},\mathbf{r}'\in\mathbf{X}'}\mathbf{r}\circ \mathbf{r}'$.
In link generation, our IBGN model follows the rule that any relation $\mathbf{r}_{n',n}$ can only be drawn from $\mathbf{c}_{n',n}$.
We have proved that a network constructed under this rule is globally consistent and complete, with proofs relegated to the Appendix~\ref{appendix:consistency}.
\begin{theorem}
\label{thm:consistency}
\text{(\textbf{Network Consistency and Completeness})}\\
A network $G_{d}$ constructed by obeying the interval relation constraints is always temporally consistent, and any possible combination of relations in $G_{d}$ can be constructed through our IBGN model.
\end{theorem}

Now, we are ready to assign relations to links. Suppose $\mathbf{a}_{n'}=A_{i}$, $\mathbf{a}_{n}=A_j$ and $\mathbf{c}_{n',n}=C_z$ ($1\leq i,j\leq M$, $1\leq z\leq 11$), a relation $\mathbf{r}_{n',n}$ on link $e_{n',n}\in E_d$ is chosen from a distribution over all possible relations in $\mathbf{c}_{n',n}$ as follows:
\begin{equation}
\nonumber
\small
\begin{split}
&\mathbf{r}_{n',n}\mid\mathbf{a}_{n'},\mathbf{a}_{n},\mathbf{c}_{n',n}\sim\text{\emph{Multinomial}}(\varphi_{i,j,z}),\\
%&\varphi_{i,j,z}\sim\text{\emph{Dirichlet}($\gamma$)},
\end{split}
\end{equation}
where
%$\gamma$ is a hyperparameter, and
$\varphi_{i,j,z}$ is the parameter vector of the multinomial distribution associated with the triplet $(A_i,A_j,C_z)$. Note that for an interval relation constraint containing only one relation (i.e. $C_{1}$-$C_{7}$), the probability of choosing that relation is always one.

It is also important to notice that the quality of the network structure plays an extremely important role in activity modeling. In our previous work~\cite{liu2016recognizing}, two variants with fixed network structures are considered: \emph{chain-based network} as in Figure~\ref{fig:example-network-structure}(a), where only the links between two neighbouring nodes are constructed in networks; \emph{fully-connected network} as in Figure~\ref{fig:example-network-structure}(b), where all pairwise links are constructed in networks. In fact, they are two special cases of our IBGN model. In chain-based networks, only a set of $\lvert V_d \rvert-1$ links $e_{1,2}, e_{2,3}, \ldots, e_{\lvert V_d \rvert-1, \lvert V_d \rvert}$ are generated, with $e_{n,n+1}$ ($1\leq n\leq \lvert V_d \rvert-1$) representing the link of two neighboring nodes $v_{n}\rightarrow v_{n+1}$. Any interval relation constraint $\mathbf{c}_{n,n+1}$ in chain-based networks equals to $\mathbb{R}$, and thus such networks are inherently consistent because no inconsistent triangle exists. However crucial relations may be missing in this model. On the other side of the spectrum, we have fully-connected networks, where all possible pairwise links are considered. Any $\mathbf{x}_{p,q}$ in fully-connected networks equals to $\mathbf{r}_{p,q}$. When fitting the IBGN model, it is possible to increase the likelihood by adding links, but doing so may result in overfitting. Instead of prefixing the network structures, in this work we relax the assumption of a structure being either fully-connected or chain-based, and consider learning an optimal structure (i.e. to decide which links should exist in the network) from data.
This allows the IBGN model to handle temporal consistency for arbitrary network structures, as presented in Figure~\ref{fig:example-network-structure}(c).
\begin{figure*}[!htbp]
\centering
\subfigure[Chain-based network.]{
\includegraphics[width=0.2\textwidth]{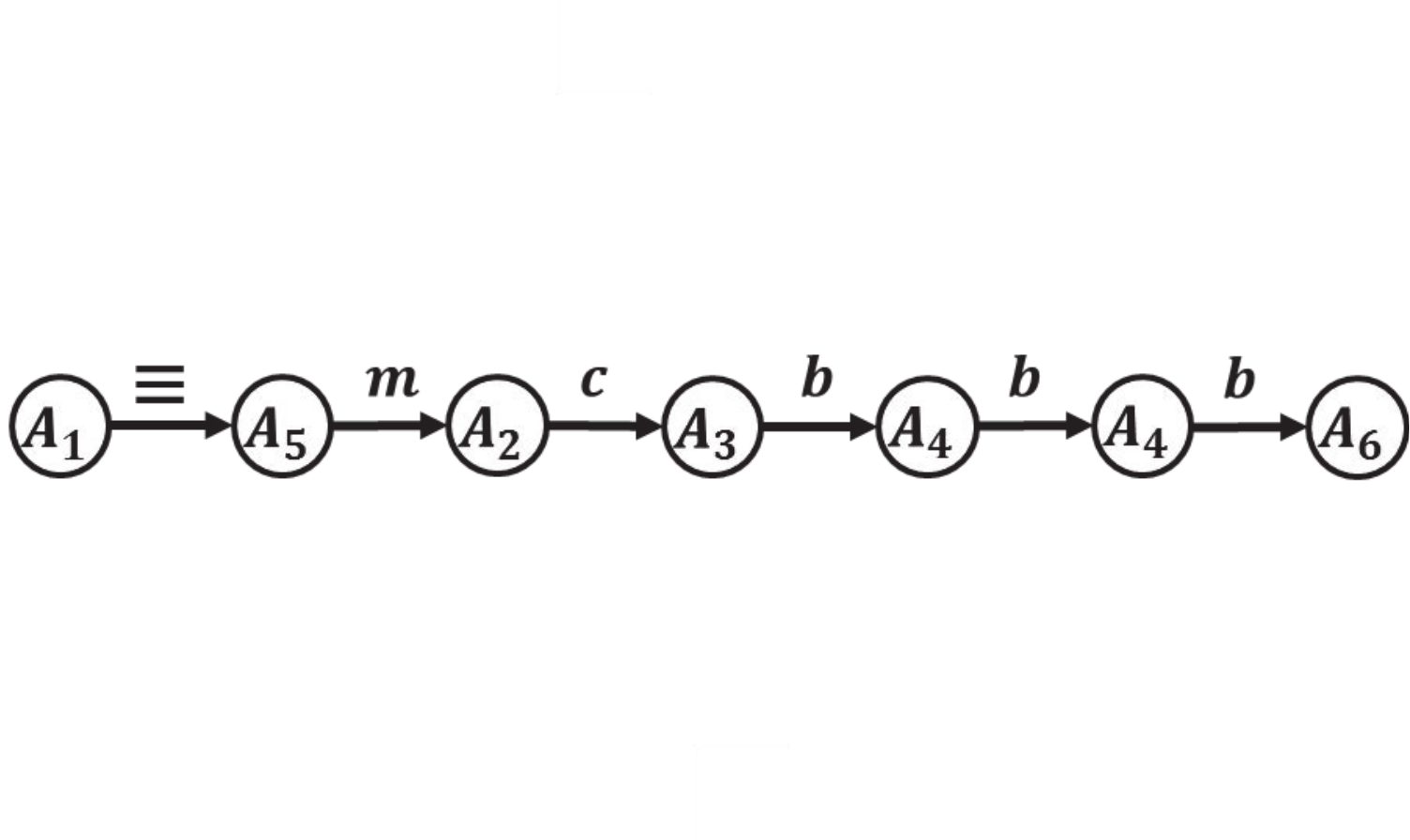}\label{fig:example-example-network-chain}
}
\subfigure[Fully-connected network.]{
\includegraphics[width=0.2\textwidth]{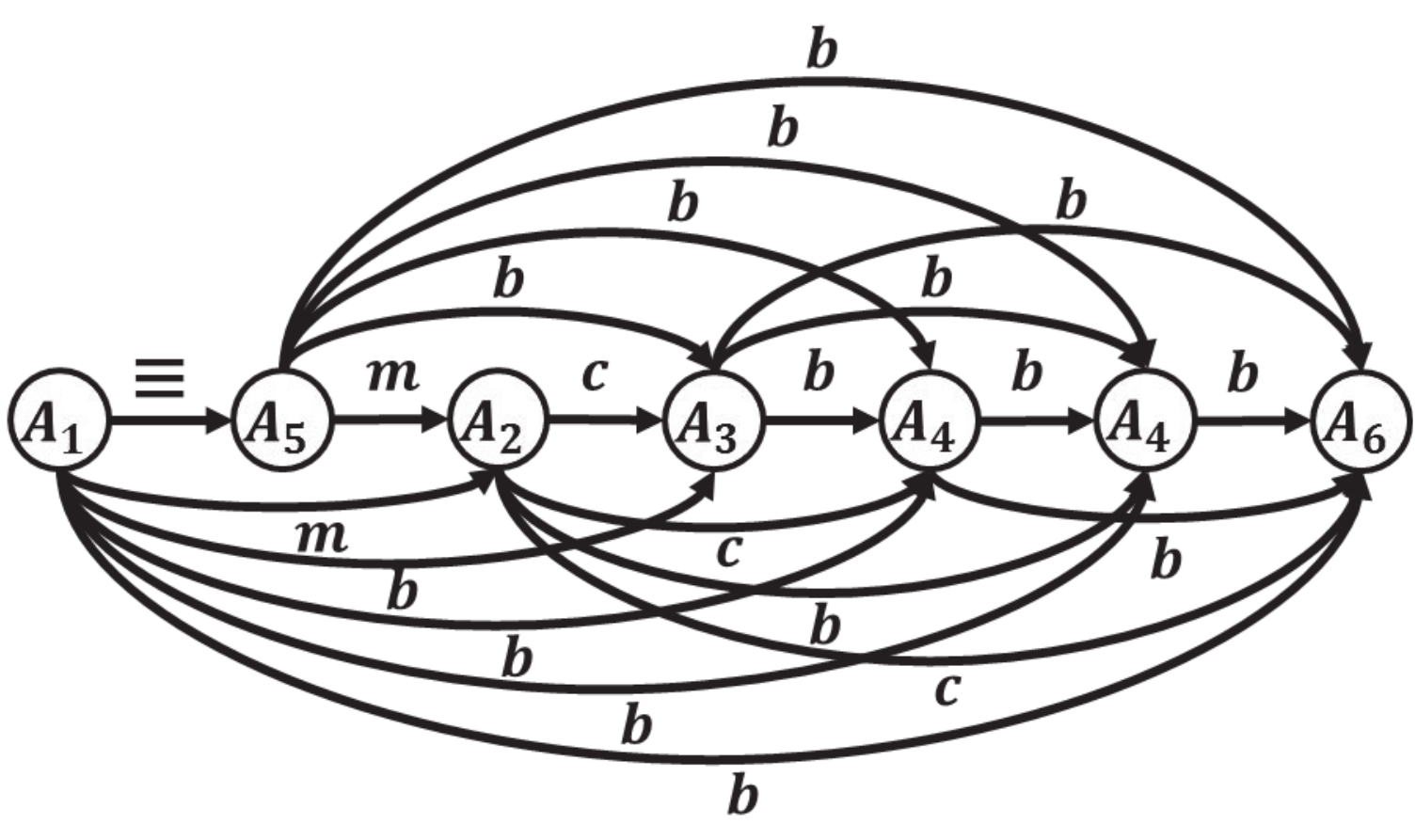}\label{fig:example-example-network-fully}
}
\subfigure[An exemplar arbitrary network structure.]{
\includegraphics[width=0.2\textwidth]{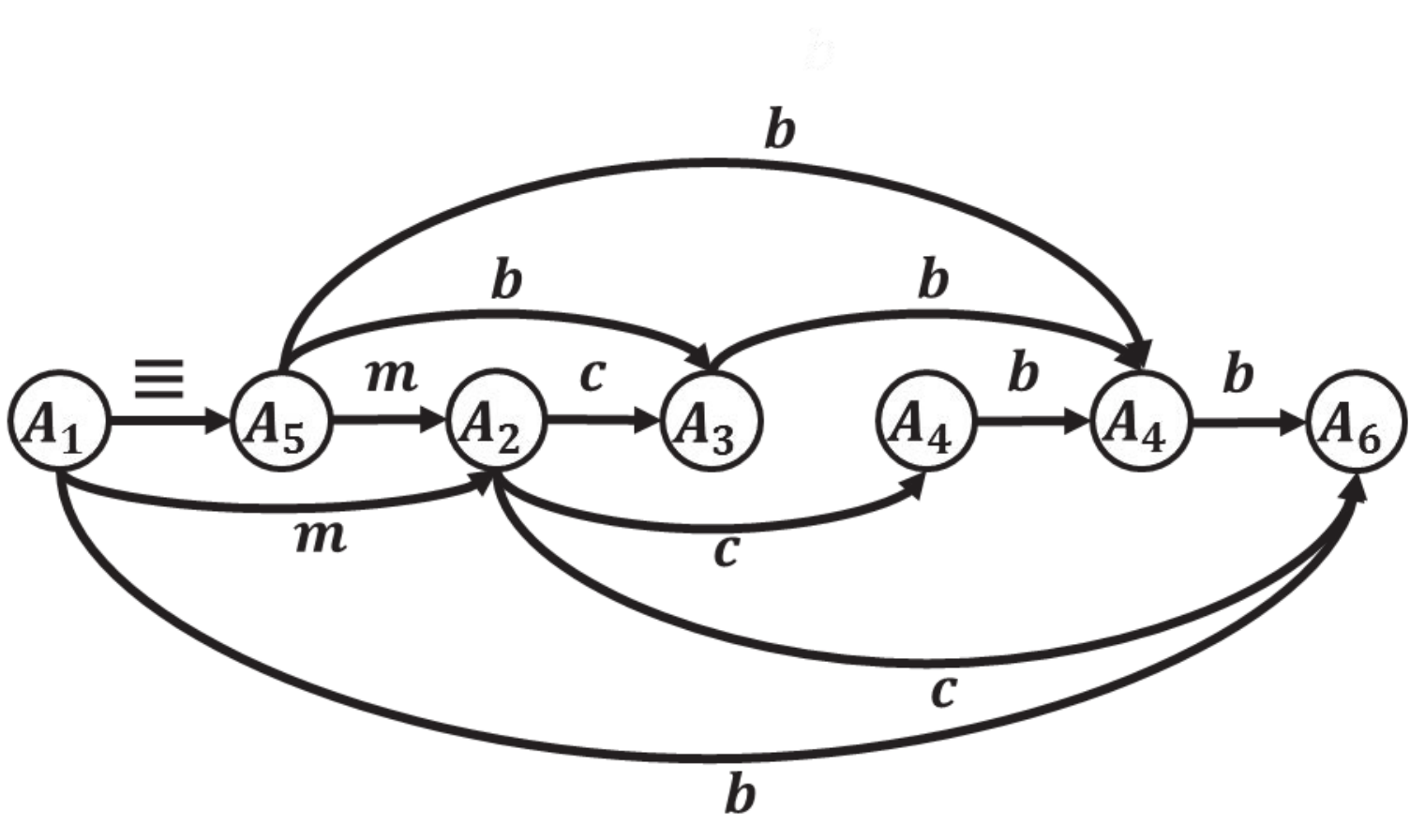}\label{fig:example-example-network-arbitrary}
}
\caption{Three possible network structures in link generation.}
\label{fig:example-network-structure}
\end{figure*}

\subsection{The Generative Process}
%Given a dataset of instances from a set of $L$ complex activities, we denote a training dataset $\mathbb{D}_{l}$ of instances from the $l$-th complex activity.
For each dataset $\mathbb{D}_{l} \subset \mathbb{D}$ containing a particular complex activity $1 \leq l \leq L$, our model assumes the whole generative process including node generation and link generation in Algorithm~\ref{algo-gp}. Notice that the optimal network structure $\mathbf{G}^*$ would be learned from $\mathbb{D}_{l}$ with details to be elaborated in section~\ref{section:structure}. The structure $\mathbf{G}^*$ demonstrates whether a link needs to be generated in the network. For example, to construct the network $G_d=(V_d,E_d)$ for a certain complex activity instance $d$, the link $e_{n',n}$ from $v_{n'}$ to $v_{n}$ is involved in $G_d$ if and only if $e_{n',n}$ obeys the structure of $\mathbf{G}^*$, denoted by $e_{n',n}\in E_d\models\mathbf{G}^*$.
\begin{algorithm}
%\vspace{-3mm}
  \caption{Generative process.}\label{algo-gp}
  \scriptsize
  \begin{algorithmic}[1]
    \Procedure{Generate-Networks}{$\mathbb{D}_{l}$}
    \State Learn an optimal network structure $\mathbf{G}^*$ from $\mathbb{D}_{l}$; \Comment{$\mathbf{G}^*=(V_{\mathbf{G}^*},E_{\mathbf{G}^*})$, }
    \State Choose a distribution $\theta_{\zeta}\sim\emph{Dirichlet}(\beta)$ ($\zeta=1,2,\ldots$);
    %\State Choose a distribution $\varphi_{i,j,z}\sim\emph{Dirichlet}(\gamma)$ ($1\leq i,j\leq M, 1\leq z\leq 11$);
    \For{each complex activity instance $d$ in $\mathbb{D}_{l}$} \Comment{construct a network $G_d=(V_d,E_d)$}
      \For{each node $v_{n}$ ($1\leq n\leq \lvert V_d \rvert$)}
        \State Choose a table $\mathbf{t}_{n}\sim\emph{CRP}(\mathbf{t}_{1},\ldots,\mathbf{t}_{n-1};\alpha)$;\Comment{Suppose $\mathbf{t}_{n}=T_\zeta$}
        \State Choose an atomic action $\mathbf{a}_{n}\mid\mathbf{t}_{n}\sim\emph{Multinomial}(\theta_{\zeta})$;\Comment{Suppose $\mathbf{a}_{n}=A_j$}
      \For{each link $e_{n',n}$ ($n-1\geq n'\geq 1$)}
          \If{$e_{n',n} \text{ obeys the structure of } \mathbf{G}^*$ (i.e. $e_{n',n}\models\mathbf{G}^*$)}
              \State Calculate $\mathbf{c}_{n',n}$ on the link $e_{n',n}$;\Comment{Suppose $\mathbf{c}_{n',n}=C_z$}
              \State Choose a relation $\mathbf{r}_{n',n}\mid\mathbf{a}_{n'},\mathbf{a}_{n},\mathbf{c}_{n',n}\sim\emph{Multinomial}(\varphi_{i,j,z})$;\Comment{Suppose $\mathbf{a}_{n'}=A_i$}
          \EndIf
      \EndFor
      \EndFor
    \EndFor
  \EndProcedure
\end{algorithmic}
\end{algorithm}

The joint distribution of variables $\mathbf{t}$, $\mathbf{a}$, and $\mathbf{r}$, is given by:
\begin{equation}
%\nonumber
\label{eq:jointprobability-I}
\footnotesize
\begin{split}
P(\mathbf{a}, \mathbf{t}, \mathbf{r};\alpha, \beta)=&\prod_{d\in \mathbb{D}_{l}}\Big(\prod_{v_n\in V_d}\big(P(\mathbf{t}_{n}\mid \mathbf{t}_{1},\ldots,\mathbf{t}_{n-1};\alpha)P(\mathbf{a}_{n}\mid \mathbf{t}_{n};\beta)\big)\\
&\prod_{\substack{n'<n,\\ e_{n',n}\models\mathbf{G}^*}}{P(\mathbf{r}_{n',n}\mid\mathbf{a}_{n},\mathbf{a}_{n'},\mathbf{c}_{n',n})}\Big).
\end{split}
\end{equation}
The total number of variables $\mathbf{t}$, $\mathbf{a}$, and $\mathbf{r}$ are ${\sum}_{d\in\mathbb{D}_{l}}\lvert V_d \rvert$, ${\sum}_{d\in\mathbb{D}_{l}}\lvert V_d \rvert$ and $\lvert\mathbb{D}_{l}\rvert\cdot{\lvert E_{\mathbf{G}^*}\rvert}$, respectively.
It is worth noting that we often set $\ell=\mathop{\max}\limits_{d\in \mathbb{D}_{l}}{\{\lvert V_d \lvert\}}$, due to the fact that given a training dataset $\mathbb{D}_l$ a number of $\mathop{\max}\limits_{d\in \mathbb{D}_{l}}{\{\lvert V_d \lvert\}}$ tables are occupied at most.
%An $\ell>\mathop{\max}\limits_{d\in \mathbb{D}_{l}}{\{\lvert V_d \lvert\}}$ is remained to generate a new group of networks with the size greater than any training types of complex activities, where atomic action types are chosen with the same probability (i.e. $\frac{1}{M}$) after the $\mathop{\max}\limits_{d\in \mathbb{D}_{l}}{\{\lvert V_d \lvert\}}$-th node.
%It can be seen that there are at most $\ell=\mathop{\max}\limits_{d\in \mathbb{D}_{l}}{\{\lvert V_d \lvert\}}$ tables involved in the above procedure.

\section{IBGN Learning}
In what follows we focus on how to learn the network structure and the parameter vectors $\theta$ and $\varphi$ from the training data $\mathbb{D}_{l}$ for a particular complex activity $l$.

\subsection{Learning Network Structure}
\label{section:structure}
Instead of configuring a network with chain-based or fully-connected links,
we would like to learn an network structure $\mathbf{G}$ in IBGN according to a score function that best matches the training data $\mathbb{D}_{l}$, i.e., learning from empirical data on which links to select in our constructed networks.

An IBGN model is built over a collection of variables $\mathbf{t}$ for table assignments, $\mathbf{a}$ for atomic-action assignments and $\mathbf{r}$ for relation assignments.
In detail, for a specific instance $d=<I_1,I_2,\ldots,I_{k}>$ with $k = \lvert V_d \rvert$, its corresponding network involves variables
$\mathbf{t}=\{\mathbf{t}_{1}, \mathbf{t}_{2}, \ldots, \mathbf{t}_{k}\}$,
$\mathbf{a}=\{\mathbf{a}_{1}, \mathbf{a}_{2}, \ldots, \mathbf{a}_{k}\}$,
$\mathbf{r}=\{\mathbf{r}_{1,2}, \mathbf{r}_{1,3}, \ldots, \mathbf{r}_{1,k}, \mathbf{r}_{2,3}, \mathbf{r}_{2,4}, \ldots, $ $\mathbf{r}_{2,k}, \ldots\ldots, \mathbf{r}_{k-1,k}\}$,
$\mathbf{c}=\{\mathbf{c}_{1,2}, \mathbf{c}_{1,3}, \ldots, \mathbf{c}_{1,k}, \mathbf{c}_{2,3}, \mathbf{c}_{2,4}, \ldots, \mathbf{c}_{2,k},$ $ \ldots\ldots, \mathbf{c}_{k-1,k}\}$.
To consider a general network structure, we first introduce a \emph{null} interval to make every instance in $\mathbb{D}_{l}$ having the same number of $k^*=\mathop{\max}\limits_{d\in\mathbb{D}_{l}} \{ \lvert V_d\rvert \}$ intervals. A \emph{null} interval is defined as $(+\infty,null,+\infty)$ such that its associated atomic action is \emph{null} and its temporal relation with any other intervals is \emph{null}. In other words, \emph{null} can be viewed as a special atomic action class.
For any instance of size $k<k^*$, $k^*-k$ \emph{null} intervals are appended at the rear of the instance. In this way, every instance has the same number of $k^*$ intervals. Correspondingly, any IBGN has a total number of $(k^*+1)k^*$ possible random variables, with $k^*$ possible variables for tables $\mathbf{t}$, $k^*$ possible variables for atomic actions $\mathbf{a}$, $\frac{k^*(k^*-1)}{2}$ possible variables for interval relations $\mathbf{r}$ and $\frac{k^*(k^*-1)}{2}$ possible variables for interval relation constraints $\mathbf{c}$.
An exemplar IBGN model can be described in Figure~\ref{fig:structure_variable}, where each $v_i$ is associated with variables $\mathbf{t}_i$ and $\mathbf{a}_i$, and each $e_{i,j}$ is associated with variable $\mathbf{r}_{i,j}$ and $\mathbf{c}_{i,j}$, with $1 \leq i \leq j \leq k^*$.
\begin{figure*}[!htbp]
\centering
\includegraphics[width=0.65\textwidth]{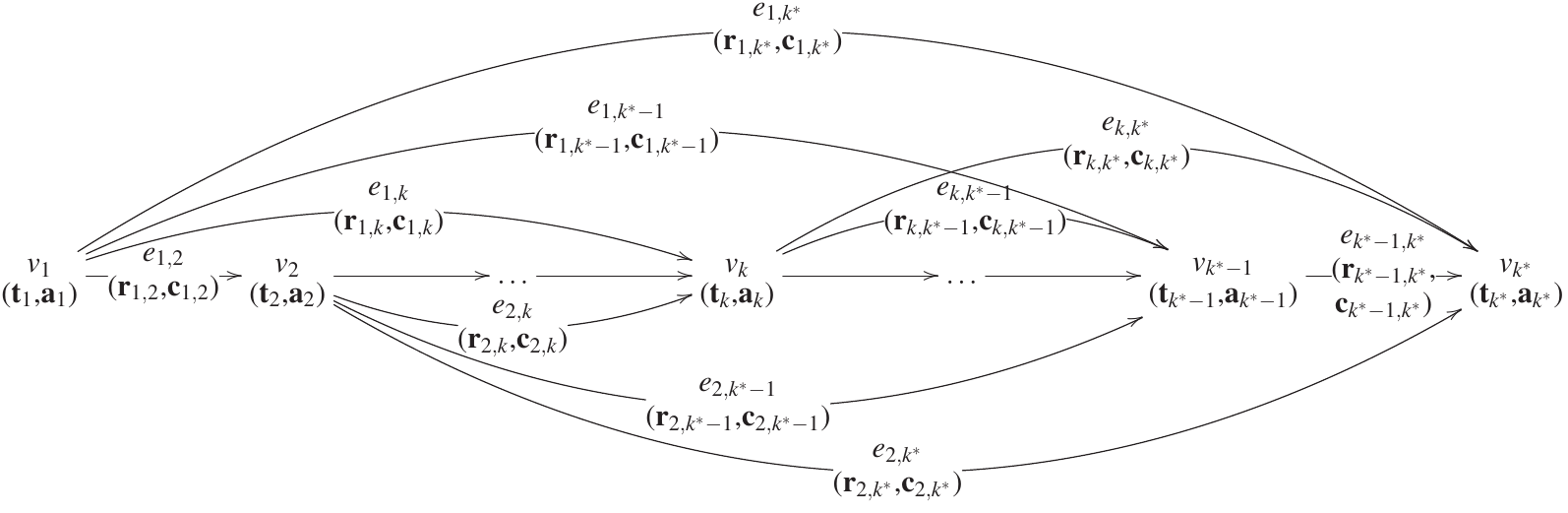}
\caption{The illustration of the IBGN network structure $\mathbf{G}$ associated with variables.}
\label{fig:structure_variable}
\end{figure*}
\begin{comment}
\begin{displaymath}
\scalebox{1}{
\xymatrix@C=2cm{
  \txt{$v_1$\\($\mathbf{t}_{1}$,$\mathbf{a}_1$)} \ar[r]|-{\txt{$e_{1,2}$\\($\mathbf{r}_{1,2}$,$\mathbf{c}_{1,2}$)}} \ar@/^2pc/[rrr]|-{\txt{$e_{1,k}$\\($\mathbf{r}_{1,k}$,$\mathbf{c}_{1,k}$)}} \ar@/^4.5pc/[rrrrr]|-{\txt{$e_{1,k^*-1}$\\($\mathbf{r}_{1,k^*-1}$,$\mathbf{c}_{1,k^*-1}$)}} \ar@/^7.5pc/[rrrrrr]|-{\txt{$e_{1,k^*}$\\($\mathbf{r}_{1,k^*}$,$\mathbf{c}_{1,k^*}$)}}
  & \txt{$v_2$\\($\mathbf{t}_{2}$,$\mathbf{a}_2$)} \ar[r]
  \ar@/_1.6pc/[rr]|-{\txt{$e_{2,k}$\\($\mathbf{r}_{2,k}$,$\mathbf{c}_{2,k}$)}} \ar@/_4pc/[rrrr]|-{\txt{$e_{2,k^*-1}$\\($\mathbf{r}_{2,k^*-1}$,$\mathbf{c}_{2,k^*-1}$)}} \ar@/_6pc/[rrrrr]|-{\txt{$e_{2,k^*}$\\($\mathbf{r}_{2,k^*}$,$\mathbf{c}_{2,k^*}$)}}
  & \ldots \ar[r]
  & \txt{$v_k$\\($\mathbf{t}_{k}$,$\mathbf{a}_k$)} \ar[r]
  \ar@/^2pc/[rr]|-{\txt{$e_{k,k^*-1}$\\($\mathbf{r}_{k,k^*-1}$,$\mathbf{c}_{k,k^*-1}$)}} \ar@/^4pc/[rrr]|-{\txt{$e_{k,k^*}$\\($\mathbf{r}_{k,k^*}$,$\mathbf{c}_{k,k^*}$)}}
  & \ldots \ar[r]
  & \txt{$v_{k^*-1}$\\($\mathbf{t}_{k^*-1}$,$\mathbf{a}_{k^*-1}$)} \ar[r]|-{\txt{$e_{k^*-1,k^*}$\\($\mathbf{r}_{k^*-1,k^*}$,\\$\mathbf{c}_{k^*-1,k^*}$)}}
  & \txt{$v_{k^*}$\\($\mathbf{t}_{k^*}$,$\mathbf{a}_{k^*}$)}
}
}
\end{displaymath}
\end{comment}
To this end, our IBGN structure learning problem is defined as to find a $\mathbf{G}=(V_{\mathbf{G}}, E_{\mathbf{G}})$ such that the score of $\mathbf{G}$ given $\mathbb{D}_{l}$ is maximum.
%where $\lvert V_{\mathbf{G}}\rvert=k^*$.

Next, we employ \emph{structure constraints} to translate the IBGN model to a corresponding problem in Bayesian networks.
\begin{definition}
\label{def:ibgn-structure}
Given an IBGN model $\mathbf{G}$, its corresponding Bayesian network is defined as a directed bipartite graph $\mathbf{G}'=(V_{\mathbf{G}'},E_{\mathbf{G}'})$ where $V_{\mathbf{G}'}=U_{\mathbf{G}'}\bigcup W_{\mathbf{G}'}$, with the structure constraints such that
\begin{equation}
\nonumber
\label{eq:bn}
%\footnotesize
\begin{split}
(1)&\lvert U_{\mathbf{G}'}\rvert=k^*,\quad \lvert W_{\mathbf{G}'}\rvert={k^*(k^*-1)}/{2},\\
(2)&\forall_{v'\in U_{\mathbf{G}'}}:{\widehat{v'}=M+1}, \quad \forall_{v'\in W_{\mathbf{G}'}}:{\widehat{v'}=8},\\
(3)&\forall_{v'\in U_{\mathbf{G}'}}:{\text{in-degree}(v')=0},\quad \forall_{v'\in W_{\mathbf{G}'}}:{\text{in-degree}(v')=0\text{ or }2},
\end{split}
\end{equation}
where $\widehat{v'}$ denotes the number of distinct elements of $v'$.
\end{definition}
A node $v'_n$ in $U_{\mathbf{G}'}$ ($1\leq n\leq \lvert U_{\mathbf{G}'}\rvert$) maps to the variable $\mathbf{a}_n$ in $\mathbf{G}$, while a node $v'_{n',n}$ in $W_{\mathbf{G}'}$ ($1\leq n'<n\leq \lvert U_{\mathbf{G}'}\rvert$) maps to the variable $\mathbf{r}_{n',n}$ in $\mathbf{G}$. That is, $U_{\mathbf{G}'}=\{\mathbf{a}_1,\ldots,\mathbf{a}_{k^*}\}$ and $W_{\mathbf{G}'}=\{\mathbf{r}_{1,2},\mathbf{r}_{1,3},\ldots,\mathbf{r}_{1,k^*},\ldots,\mathbf{r}_{k^*-1,k^*}\}$. Notice that a \emph{null} is introduced to represent the absence of a node or a relation in an instance of $\mathbb{D}_{l}$ (constraint (2)), and thus there are $M+1$ possible atomic-action assignments for a node in $U_{\mathbf{G}'}$ and $\lvert\mathbb{R}\rvert+1=8$ possible relation assignments for a node in $W_{\mathbf{G}'}$. Moreover, any node $v'_n$ in $U_{\mathbf{G}'}$ has no parent, and any node $v'_{n',n}$ in $W_{\mathbf{G}'}$ has either being connected to the nodes $v_{n'}$ and $v_{n}$ in $U_{\mathbf{G}'}$ or not being connected to any node (constraint (3)). That means a link $e_{n',n}$ exists in $\mathbf{G}$ if and only if its corresponding node $v'_{n',n}$ in $\mathbf{G}'$ has $\text{in-degree}(v'_{n',n})=2$. The structure of $\mathbf{G}'$ associated with the variables is illustrated in Figure~\ref{fig:structure_bipartite}.
%More descriptions about the transformation of learning a structure of $\mathbf{G}$ to learning a structure of $\mathbf{G}'$ are provided in our technical report.
\begin{figure}[!htbp]
\centering
\includegraphics[width=\columnwidth]{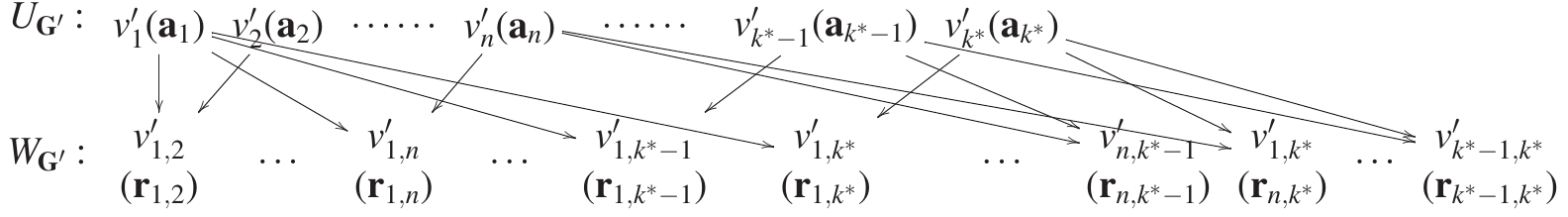}
\caption{The corresponding bipartite graph $\mathbf{G}'$ of the IBGN structure $\mathbf{G}$.}
\label{fig:structure_bipartite}
\end{figure}
\begin{comment}
\begin{displaymath}
%\footnotesize
\scalebox{0.37}{
\LARGE
\xymatrix@C=0.2cm{
  U_{\mathbf{G}'}:
  & \txt{$v'_1$($\mathbf{a}_1$)} \ar[d] \ar[rrd] \ar[rrrrd] \ar[rrrrrd]
  & \txt{$v'_2$($\mathbf{a}_2$)} \ar[ld]
  & \ldots\ldots
  & \txt{$v'_n$($\mathbf{a}_{n}$)} \ar[ld] \ar[rrrrd] \ar[rrrrrd]
  & \ldots\ldots
  & \txt{$v'_{k^*-1}$($\mathbf{a}_{k^*-1}$)} \ar[ld] \ar[rrd] \ar[rrrrrd]
  & \txt{$v'_{k^*}$($\mathbf{a}_{k^*}$)} \ar[ld] \ar[rrd] \ar[rrrrd]\\
  W_{\mathbf{G}'}:
  & \txt{$v'_{1,2}$\\($\mathbf{r}_{1,2}$)}
  & \ldots
  & \txt{$v'_{1,n}$\\($\mathbf{r}_{1,n}$)}
  & \ldots
  & \txt{$v'_{1,k^*-1}$\\($\mathbf{r}_{1,k^*-1}$)}
  & \txt{$v'_{1,k^*}$\\($\mathbf{r}_{1,k^*}$)}
  & \ldots
  & \txt{$v'_{n,k^*-1}$\\($\mathbf{r}_{n,k^*-1}$)}
  & \txt{$v'_{1,k^*}$\\($\mathbf{r}_{n,k^*}$)}
  & \ldots
  & \txt{$v'_{k^*-1,k^*}$\\($\mathbf{r}_{k^*-1,k^*}$)}
}
}
\end{displaymath}
\end{comment}

Now, the problem of determining whether a link should exist in IBGN is converted to the problem of finding a set of links $E_{\mathbf{G}'}$ with the maximal score under the above constraints. In particular, we consider the Bayesian Information Criterion ($\mathrm{BIC}$) as the score function, which addresses the overfitting issue by introducing a penalty term in the structure, as proved in the Appendix~\ref{appendix:bic}:
\begin{equation}
\label{eq:bic}
%\nonumber
%\scriptsize
\small
\begin{split}
&\text{BIC-Score}(\mathbf{G}:\mathbb{D}_{l})=\text{BIC-Score}(\mathbf{G}':\mathbb{D}_{l})+\lambda\\
%&=\mathop{\max}\limits_{\Phi}{L_{\mathbf{G}',\mathbb{D}_{l}}(\Phi)}-\frac{\log\lvert\mathbb{D}_{l}\rvert}{2}\cdot\lvert\Phi\rvert+\lambda\\
&=\mathop{\max}\limits_{\Phi}{\log\mathop{\prod}\limits_{i=1}^{\lvert W_{\mathbf{G}'}\rvert}\mathop{\prod}\limits_{j=1}^{\widehat{\Pi_i}}\mathop{\prod}\limits_{k=1}^{8}}{\phi_{ijk}^{n'_{ijk}}}
-\frac{\log\lvert\mathbb{D}_{l}\rvert}{2}\cdot\mathop{\sum}\limits_{i=1}^{\lvert W_{\mathbf{G}'}\rvert}{7\cdot\widehat{\Pi_{i}}}+\lambda,
\end{split}
\end{equation}
where $\lambda$ is a constant.
%The likelihood $L_{\mathbf{G}',\mathbb{D}_{l}}$ is equal to the probability of those observed atomic-action and interval-relation assignments for the nodes in $\mathbf{G}'$ given $\mathbb{D}_{l}$.
$\Pi_i$ denotes the parents of the $i$-th node in $W_{\mathbf{G}'}$ and $\widehat{\Pi_i}=(M+1)^{\lvert \Pi_i\rvert}$, which is the number of possible instantiations of the parent set $\Pi_i$ of the $i$-th node. In fact, $\widehat{\Pi_i}=1 \text{ or } (M+1)^2$. For example, suppose $v'_{n',n}$ be the $i$-th node in $W_{\mathbf{G}'}$ and $v'_{n'}, v'_n\in U_{\mathbf{G}'}$. If the links $v'_{n'}\rightarrow v'_{n',n}$ and $v'_{n}\rightarrow v'_{n',n}$ exist in $E_{\mathbf{G}'}$ (i.e. $\text{\emph{in-degree}}(v'_{n',n})=2$), then the node $v'_{n',n}$ has two parents (i.e. $\Pi_i=\{v'_{n'}, v'_{n}\}$) whose number of categories is $(M+1)$, and thus $\widehat{\Pi_i}=(M+1)^2$; otherwise, $\widehat{\Pi_i}=1$.
In addition, $\Phi$ is the parameter vector such that $\forall_{ijk}:\phi_{ijk}=P(x_{ik}\mid \pi_{ij})$, where $x_{ik}$ and $\pi_{ij}$ denotes that the $i$-th node in $W_{\mathbf{G}'}$ is assigned with the $k$-th element and its parent nodes are assigned with the $j$-th element (i.e. an instantiation of its parent set $\Pi_i$), respectively; $n'_{ijk}$ indicates how many instances of $\mathbb{D}_{l}$ contain both $x_{ik}$ and $\pi_{ij}$. Note that any node $v'_{n',n}\in W_{\mathbf{G}'}$ has eight elements (i.e. $\widehat{v'_{n',n}}=8$).
At this stage, several techniques can be employed to learn the structure of $\mathbf{G}'$ efficiently~\cite{de2011efficient,fan2015improved,liu2015multipe}. After finding the best structure $\mathbf{G}'^*=\mathop{\arg\max}\limits_{\mathbf{G}'}{\text{BIC-Score}(\mathbf{G}':\mathbb{D}_{l})}$, a link $e_{n',n}$ is in $\mathbf{G}$ if and only if $\text{\emph{in-degree}}(v'_{n',n}\in W_{\mathbf{G}'^*})=2$.

\subsection{Learning Parameters}
We first estimate the parameters $\theta=\{\theta_{1}, \theta_{2}, \ldots, \theta_{\ell}\}$ for node generation. Since the variable $\mathbf{t}_n$ is latent for each node $v_n$ in our generative process, we shall approximately estimate the posterior distribution on $\mathbf{t}$ by Gibbs sampling. Formally, we marginalize the joint probability in Eq.\eqref{eq:jointprobability-I} and derive the posterior probability of latent variable $\mathbf{t}_n$ being assigned to the table $T_{\zeta}$ ($1\leq\zeta\leq\ell$) as follows:
\begin{equation}
\scriptsize
\label{eq:gibbssampling}
\nonumber
\begin{split}
&P(\mathbf{t}_{n}=T_{\zeta}\mid\mathbf{t}_{-n}, \mathbf{a}, \mathbf{r};\alpha, \beta)
\propto\frac{na_{\zeta,\tilde{\mathbf{i}},-n}+\beta_{\zeta,\tilde{\mathbf{i}}}}{\sum_{i'=1}^{M}{(na_{\zeta,i',-n}+\beta_{\zeta,i'})}}
\times \left\{
  \begin{array}{l l}
    \frac{nt_{\zeta}}{n+\alpha_{\zeta}-1} & \quad \text{if $\zeta\leq NT_{n}$}\\
    \frac{\alpha_{\zeta}}{n+\alpha_{\zeta}-1} & \quad \text{if $\zeta=NT_{n}+1$}
  \end{array}
  \right.
\end{split}
\end{equation}
where $na_{\zeta,i'}$ is the count that nodes have been assigned to the atomic action type $A_{i'}$ at the table $T_\zeta$, and $nt_{\zeta}$ is the count of nodes in $G_d$ that has been assigned to the table $T_\zeta$. $\tilde{\mathbf{i}}$ refer to the atomic action assignments for nodes $v_n$. $NT_{n}$ is the count of occupied tables with $\sum_{i=1}^{NT_{n}}nt_{i}=n-1$. The suffix $-n$ of $na$ means the count that does not include the current assignment of table for the node $v_{n}$. $\alpha_{\zeta}$ is the tuning parameter for the $\zeta$-th table selection during \emph{CRP}; $\beta_{\zeta,i}$ is the Dirichlet prior for the $i$-th atomic action conditioned on the $\zeta$-th table. We provide the detailed derivations of the Gibbs sampling in Appendix~\ref{appendix:sampling}. By sampling the latent tables following the above distribution, the distributions of $\theta_{\zeta}$ ($1\leq\zeta\leq \ell$) can be estimated as
\begin{equation}
%\scriptsize
\small
\nonumber
\begin{split}
\theta^*_{\zeta,i}=\frac{na_{\zeta,i}+\beta_{\zeta,i}}{\sum_{i'=1}^{M}{(na_{\zeta,i'}}+\beta_{\zeta,i'})}&~\quad \text{, for each }1\leq i\leq M.
\end{split}
\end{equation}

%\section{Estimating Dirichlet hyperparameters}
Normally, the hyperparameters $\alpha$ and $\beta$ are set as fixed values before the execution of a Gibbs sampler.
In our IBGN model, $\alpha$ and $\beta$ involve a number of $\ell$ and $\ell M$ prior parameters, respectively.
As they are unfortunately unknown beforehand, it is difficult to manually encode each parameter to proper values.
As a result, we need to instead learn each hyperparameter to obtain reasonable results.
The adoption of Gibbs sampling enables us to seamlessly incorporate the tuning of these hyperparameters as presented in Algorithm~\ref{alg-paratune}. \begin{algorithm}
  \caption{Gibbs Sampling Algorithm with Hyperparameter Tuning Method Embedded.}
  \label{alg-paratune}
  \scriptsize
  \begin{algorithmic}[1]
  \Procedure{Gibbs\_Sampler\_with\_Hyperparameter\_Estimate}{}
  \State $s = 0$; \Comment{The initial iteration of the Gibbs sampler.}
  \State Initialize the values of $\alpha$ and $\beta$;
  \Repeat
  \State $s \leftarrow s+1$; \Comment{The $s$-th iteration of the Gibbs sampler.}
  \State Get the samples of latent tables generated by the Gibbs sampler at current iteration;
  \State Update hyperparameters $\alpha^{new}=f(\alpha)$, $\beta^{new}=g(\beta)$;
  \Until{termination conditions are reached;}
  \EndProcedure
\end{algorithmic}
\end{algorithm}
The stop condition may be that a predefined max iterations has been reached or that an estimation function converges to a given threshold.
%In this algorithm, the table assignments are available as a sample after each iteration. One of the most exact approaches to learn Dirichlet parameters is iterative approximation of maximum likelihood estimation by using counts of table assignments of samples. Note that \emph{CRP} is an equivalent view of the Dirichlet process.
To update the hyperparameters $\alpha$ and $\beta$, a convergent method proposed by Minka~\cite{minka2000estimating} is used as follows:
\begin{equation}
\nonumber
\footnotesize
\begin{split}
&\alpha^{(s+1)}_{\zeta}=f(\alpha_{\zeta})=\alpha_{\zeta}
\times\frac{\sum_{s}\Psi(\sum_{d\in\mathbb{D}_l}nt^{(s)}_{\zeta}+\alpha_{\zeta})-\Psi(\alpha_{\zeta})}{\sum_{s}\Psi(\sum_{d\in\mathbb{D}_l}\lvert V_d\rvert+\sum_{\zeta'=1}^{\ell}{\alpha_{\zeta'}})-\Psi(\sum_{\zeta'=1}^{\ell}{\alpha_{\zeta'}})},\\
&\beta^{(s+1)}_{\zeta,i}=g(\beta_{\zeta,i})=\beta_{\zeta,i}
\times\frac{\sum_{s}\Psi(na^{(s)}_{\zeta,i}+\beta_{\zeta,i})-\Psi(\beta_{\zeta,i})}{\sum_{s}\Psi(\sum_{i'=1}^{M}(na^{(s)}_{\zeta,i'}+\beta_{\zeta,i'}))-\Psi(\sum_{i'=1}^{M}\beta_{\zeta,i'})},\\
%&\gamma^{(s+1)}_{i,j,z,r}=h(\gamma_{i,j,z,r})=\gamma_{i,j,z,r}\times\frac{\sum_{s}\Psi(nr^{(s)}_{i,j,z,r}+\gamma_{i,j,z,r})-\Psi(\gamma_{i,j,z,r})}{\sum_{s}\Psi(\sum_{r'=1}^{\lvert C_z\rvert}(nr^{(s)}_{i,j,z,r'}+\gamma_{i,j,z,r'}))-\Psi(\sum_{r'=1}^{\lvert C_z\rvert}\gamma_{i,j,z,r'})},
\end{split}
\end{equation}
where $\Psi$ is digamma function, and the superscript $(s)$ indicates the sample generated by the Gibbs sampler at the $s$-th iteration.

Next, we estimate the parameters $\varphi=\{\varphi_{i,j,z}:1\leq i,j\leq M, 1\leq z\leq\lvert\mathbb{C}\rvert\}$ for link generation.
It can be seen that the probability distribution of variable $\mathbf{r}_{n',n}$ relies on the triplet $(A_i,A_j,C_z)$ only (where $\mathbf{a}_{n'}=A_i$, $\mathbf{a}_{n}=A_j$ and $\mathbf{c}_{n',n}=C_z$), and thus we can learn these parameters by maximum likelihood estimate method.
In our IBGN model, given a triplet $(A_i,A_j,C_z)\in\mathbb{A}\times\mathbb{A}\times\mathbb{C}$, the conditional probability distribution on $\mathbf{r}_{n',n}$ is a multinomial over all possible relations in $C_{z}$.
Then, the likelihood of the parameter $\varphi_{i,j,z}$ for $P(\mathbf{r}_{n',n}\mid A_i,A_j,C_z)$ with respect to $\mathbb{D}_l$ is:
\begin{equation}
\label{eq:model-I-MLE}
%\scriptsize
\small
\begin{split}
&L(\varphi_{i,j,z};\mathbb{D}_l)=\prod_{d\in\mathbb{D}_l}{P(\mathbf{r}_{n',n}\mid A_i,A_j,C_z;\varphi_{i,j,z})}=\prod_{r=1}^{\lvert C_z\rvert}{\varphi_{i,j,z,r}^{nr_{i,j,z,r}}},
%&L(\mathbb{D}_l;\varphi_{i,j,z})=\prod_{d\in\mathbb{D}_l}{P(\mathbf{r}_{n',n}\mid A_i,A_j,C_z;\varphi_{i,j,z})}=\frac{\Gamma(\sum_{r=1}^{\lvert C_z\rvert}{\gamma_{i,j,z,r}})}{\Gamma(\sum_{r=1}^{\lvert C_z\rvert}{(nr_{i,j,z,r}+\gamma_{i,j,z,r}})}\prod_{r=1}^{\lvert C_z\rvert}{\frac{\Gamma(\varphi_{i,j,z,r}+)}{}\varphi_{i,j,z,r}^{nr_{i,j,z,r}}},
\end{split}
\end{equation}
%Moreover, we can use the maximum likelihood estimate (MLE) to learn parameters $\varphi$ because the probability distribution of the relation $\mathbf{r}_{i,j}$ only relies on the pair of atomic actions $(A_i,A_j)$, and thus we have:
By applying a Lagrange multiplier to ensure $\sum_{r=1}^{\lvert C_{z}\rvert}{\varphi_{i,j,z,r}}=1$, maximum likelihood estimate for $\varphi_{i,j,z,r}$ is\\
\begin{equation}
\label{eq:model-I-MLE}
%\scriptsize
\small
\begin{split}
%&\varphi^*_{i,j,z,r}=\frac{nr_{i,j,z,r}+\gamma_{i,j,z,r}}{\sum_{r'=1}^{\lvert C_{z}\rvert}{(nr_{i,j,z,r'}}+\gamma_{i,j,z,r'})},
&\varphi^*_{i,j,z,r}=\frac{nr_{i,j,z,r}}{\sum_{r'=1}^{\lvert C_{z}\rvert}{nr_{i,j,z,r'}}},
\end{split}
\end{equation}
where $nr_{i,j,z,r}$ is the number of links $e_{n',n}$ are labeled with the $r$-th relation in $\mathbb{D}_{l}$, with $\mathbf{a}_{n'}=A_i$, $\mathbf{a}_{n}=A_j$, $\mathbf{c}_{n',n}=C_{z}$ and $e_{n',n}\models\mathbf{G}^*$.
%$\gamma_{i,j,z,r}$ is the Dirichlet prior for the distribution of the $r$-th relation conditioned on the triplet $(A_i,A_j,C_z)$.
Note the trivial cases of $\varphi^*_{i,j,z,r}=1$ for $1\leq z\leq 7$ as each of them contains only one element as indicated in Figure~\ref{fig:union}.
%The update formula is therefore applied to only those $\varphi^*_{i,j,z,r}$ entries with $8\leq z\leq 11$.

Now, by integrating out the latent variable $\mathbf{t}$ with all the parameters derived above, the probability of the occurrence of a new instance given the $l$-th type of complex activity is estimated below
\begin{equation}
%\scriptsize
\small
\begin{split}
&P(\mathbf{a}',\mathbf{r}'; \mathbb{D}_{l})=\prod_{\substack{i}}(\sum_{\zeta}{\theta_{\zeta,i}})\times\prod_{\substack{i,j,z,r\models\mathbf{G}^*}}{\varphi_{i,j,z,r}},
\end{split}
\end{equation}
where $\mathbf{a}'$ and $\mathbf{r}'$ are the sets of atomic actions and their relations in the new instance, respectively, and $i,j,z,r\models\mathbf{G}^*$ indicate only these links obeying the structure of $\mathbf{G}^*$ are counted. To predict which type of complex activity a new instance belongs to, we simply evaluate the posterior probabilities over each of the $L$ possible types of complex activities as
\begin{equation}
%\scriptsize
\small
\begin{split}
l^*=\mathop{\arg\max}\limits_{1\leq l\leq L}{P(\mathbf{a}',\mathbf{r}'; \mathbb{D}_{l})}.
\end{split}
\end{equation}

%\section{Our Hand-Action Dataset}
\section{Experiments}
Experiments are carried out on three benchmark datasets as well as our in-house dataset on recognizing complex hand activities.
In addition to the proposed approach \textbf{IBGN}, two variants with fixed network structures are also considered: One is \textbf{IBGN-C} for chain-based structures, where only the links between two neighbouring nodes in networks are constructed;
The other one is \textbf{IBGN-F} for fully-connected structures, where all pairwise links in networks are constructed.
Several well-established models are employed as the comparison methods, which include IHMM~\cite{modayil2008improving}, dynamic Bayesian network (DBN)~\cite{hu2008cigar} and ITBN~\cite{zhang2013modeling}, where IHMM and DBN are implemented on our own, and ITBN is obtained from the authors. All internal parameters are tuned for best performance for a fair comparison. The standard evaluation metric of \emph{accuracy} is used, which is computed as the proportion of correct predictions.

\paragraph{Experimental Set-Ups}
The Raftery and Lewis diagnostic tool~\cite{raftery1992many} is employed to detect the convergence of the Gibbs sampler (Algorithm~\ref{alg-paratune}) for the IBGN family. It has been observed that overall we have a short burn-in period, which suggests the Markov chain samples are mixing well.
Thus $nt$ and $na$ are set to the averaged counts of their first $1000$ samples after convergence.
In addition, we utilize the branch-and-bound algorithm~\cite{de2011efficient} for constraints-based structure learning. This approach can strongly reduce the time and memory costs for learning Bayesian network structures based on the BIC score function (Eq.~\eqref{eq:bic}) without losing global optimality guarantees.
Besides, to avoid the division-by-zero issue in practice (i.e. $\sum_{r'=1}^{\lvert C_z\rvert}{nr_{i,j,z,r'}}=0$ in Eq.~\eqref{eq:model-I-MLE}), we instead use $\varphi^*_{i,j,z,r}=\frac{nr_{i,j,z,r}+\rho}{\sum_{r'=1}^{\lvert C_{z}\rvert}{nr_{i,j,z,r'}}+\rho\lvert C_z\rvert}$ by introducing a small smoothing constant $\rho$ ($\rho=10^{-5}$) in the following experiments.%, which may appear in the training datasets.

\paragraph{Time Complexity Analysis}
The time complexities of IBGN-C and IBGN-F are $O(M^2+\lvert\mathbb{D}_{l}\rvert \mathrm{T}_n {\ell}^2)$ for training each complex activity category, where $\mathrm{T}_n$ is the number of iterations executed in Algorithm~\ref{alg-paratune}. IBGN has an extra time complexity of $O(\sum_{p=0}^{\log_2{K}}\binom{K}{p})$ for structure learning, where $K=\frac{k^*(k^*-1)}{2}$. On the other hand, the time complexities of the IBGN family at the testing stage are the same, which is $O( M \ell)$ for a single test instance.

\subsection{Experiments on Three Existing Benchmark Datasets}
\setcounter{paragraph}{0}
\paragraph{Datasets and Preprocessing}
The three publicly-available complex activity recognition datasets collected from different types of cameras and sensors are considered, as summarized in Table.~\ref{tab-experiment-I-property}. We employ these datasets in our evaluation due to their distinctive properties: The OSUPEL dataset~\cite{brendel2011probabilistic} can be used to evaluate the case where only a handful of atomic action types and simple relations are recorded; Opportunity~\cite{roggen2010collecting} is challenging as it contains a large number of atomic action types and also involves intricate interval relations in instances; CAD14~\cite{lillo2014discriminative} represents the datasets having relatively larger number of complex activity categories.
\begin{table}[!htbp]
\caption{Summary of the publicly available datasets.}
\scriptsize
\centering
\scalebox{0.85}{
\begin{tabular}{|p{0.28\columnwidth}|p{0.24\columnwidth}|p{0.24\columnwidth}|p{0.24\columnwidth}|}
\hline
  & OSUPEL & Opportunity & CAD14\\
\hline
Application type & Basketball play & Daily living & Composable activities\\
Recording devices & one ordinary camera & 72 on-body sensors of 10 modalities & one RGB-D camera\\
E.g. of atomic actions & shoot, jump, dribble, etc.& sit, open door, wash cup, etc.& clap, talk phone, walk, etc.\\
E.g. of complex activities & two offensive play types & relax, cleanup, coffee time, early morning, sandwich time & talk and drink, walk while clapping, talk and pick up, etc.\\
\hline
No. of atomic action types & 6 & 211 & 26\\
No. of complex activity types & 2 & 5 & 16\\
No. of instances & 72 & 125 & 693\\
No. of intervals per instance & 2-5 & 1-78 & 3-11 \\
\hline
\end{tabular}
}
\label{tab-experiment-I-property}
\end{table}

To recognize atomic actions in each dataset, we adopt the methods developed in their respective corresponding work.
That is, we employed the dynamic Bayesian network models that are used to model and recognize each atomic action including shoot, jump, dribble and so on for OSUPEL~\cite{zhang2013modeling}. Similarly, for CAD14~\cite{lillo2014discriminative} we adopted the hierarchical discriminative model to recognize atomic action such as clap, talk phone and so on through an discriminative pose dictionary. For the sensor-based Opportunity dataset, we utilized the activity recognition chain system (ARC) ~\cite{bulling2014tutorial} to recognize atomic action recognition from sensors. It is worth mentioning that we can also recognize \emph{null} type of complex activity by labeling the intervals that are not annotated to any activities in the datasets.

\paragraph{Comparison under an Ideal Condition}
First of all, the competing models are evaluated under the condition that all intervals are correctly detected.
Table~\ref{tab-experiment-I-comparison-precision} displays the averaged accuracy results over 5-fold cross-validations,
where the proposed IBGN family clearly outperforms other methods with a big margin on all three datasets.
The reason is that IBGNs engage the rich interval relations among atomic actions. Although ITBN can also encode relations, it however fails with the multiple occurrences of the same atomic actions or when inconsistent relations existing among training instances. As a considerable amount of multiple occurrences of the same atomic actions and inconsistent temporal relations exist in CAD14, ITBN performs the worst.
It can also be seen that IBGN-F with fully-connected relations performs better than IBGN-C on the Opportunity and CAD14 datasets where relations are intricate. However, IBGN-F might be overfitted when relations are simple, e.g. the OSUPEL dataset.
Overall IBGN outperforms its two variants by its ability to adaptively learn network structures from data, where fixed structures might face the issue of either overfitting or underfitting.
\begin{table}[!htbp]
\caption{Accuracy comparison on different datasets.}
\label{tab-experiment-I-comparison-precision}
\begin{center}
\scriptsize
\begin{tabular}{|c|c|c|c|c|c|c|}
%\begin{tabular}{clcc}
\hline
&\multicolumn{1}{c|}{IHMM}&\multicolumn{1}{c|}{DBN}&\multicolumn{1}{c|}{ITBN}&\multicolumn{1}{c|}{IBGN-C}&\multicolumn{1}{c|}{IBGN-F}&\multicolumn{1}{c|}{IBGN}\\
\hline
OSUPEL&0.53&0.58&0.69&0.79&0.76&\textbf{0.81}\\
\hline
Opportunity&0.74&0.83&0.88&\textbf{0.98}&0.96&\textbf{0.98}\\
\hline
CAD14&0.93&0.95&0.51&0.97&\textbf{0.98}&\textbf{0.98}\\
\hline
\end{tabular}
\end{center}
\end{table}
\begin{comment}
\begin{table}[!htbp]
\caption{Accuracy comparison on different datasets.}
\label{tab-experiment-I-comparison-precision}
\begin{center}
\scriptsize
\begin{tabular}{|c|c|c|c|c|c|c|c|}
%\begin{tabular}{clcc}
\hline
&\multicolumn{1}{c|}{HMM}&\multicolumn{1}{c|}{CHMM}&\multicolumn{1}{c|}{DBN}&\multicolumn{1}{c|}{ITBN}&\multicolumn{1}{c|}{IBGN-C}&\multicolumn{1}{c|}{IBGN-F}&\multicolumn{1}{c|}{IBGN}\\
\hline
OSUPEL&0.53&0.75&0.58&0.69&0.79&0.76&0.81\\
\hline
Opportunity&0.77&0.85&0.80&0.88&0.85&0.87&0.84\\
\hline
CAD14&0.93&0.96&0.95&0.51&0.97&0.98&0.98\\
\hline
\end{tabular}
\end{center}
\end{table}
\end{comment}

\paragraph{Robustness Tests under Atomic-Level Errors}
In practice the accuracy of atomic action recognition will significantly affect complex activity recognition results.
To evaluate the performance robustness, we also compare the competing models under atomic action recognition errors.
First, it is important to check whether our model is robust under label perturbations of atomic-action-level (or atomic-level for short). To show this, we synthetically perturb the atomic-level predictions.
%by flipping on ground-truth annotations with specific amount of percentages.
Figure~\ref{fig:system_comparison_with_errors} reports the comparison results on Opportunity under two common atomic-level errors.
It can be seen that IBGNs are more robust to misdetection errors where atomic actions are not detected or are falsely recognized as another actions.
We perturbed the true labels with error rates ranging from 10-30 percents to simulate synthetic misdetection errors.
Similarly, we also perturbed the start and end time of intervals with noises of 10-30 percents to simulate duration-detection errors where interval durations are falsely detected.
It is also clear that IBGNs outperform other competing models under duration-detection errors.

In addition, we report the evaluated performances under real detected errors caused by the ARC system for atomic-level recognition.
We chose three classifiers for atomic action recognition from the ARC system, i.e. kNN, SVM and DT.
Features such as mean, variance, correlation and so on are selected by setting a time-sliced window of 1s.
After classification, each interval is assigned to an atomic action type.
As shown in Figure~\ref{fig:system_comparison_with_errors}, the models which can manage interval relations are relatively more robust to the atomic-level errors than other models, such as ITBN and IBGNs.
Moreover, it is evident that IBGNs are noticeably more robust than ITBN with around 15\% -- 87\% performance boost.
IBGN performs the best among its family because it is more capable of handling the structural variability in complex activities than the other two variants, which may avoid more noise existing in training and testing information. Note similar conclusions are also obtained on the OSUPEL and CAD14 datasets.
\begin{figure}[!htbp]
\centering
\includegraphics[width=0.85\columnwidth]{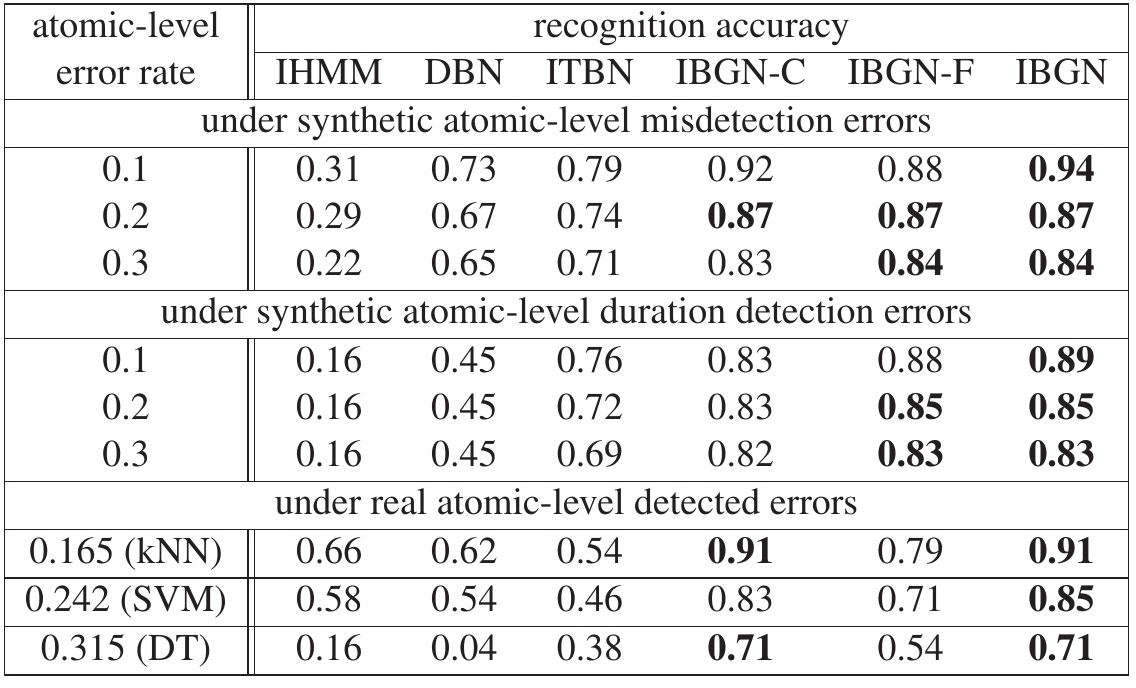}
\caption{Accuracies under atomic-level errors on Opportunity dataset.}
\label{fig:system_comparison_with_errors}
\end{figure}
\begin{comment}
\begin{table}[!htbp]
\caption{Accuracies under atomic-level errors on Opportunity dataset.}
\label{tab-system-comparison-with-errors}
\begin{center}
\begin{threeparttable}
\scriptsize
\begin{tabular}{|c||cccccc|}
%\begin{tabular}{clcc}
\hline
%\multirow{2}{*}{\scriptsize{atomic action error rate}}&\multicolumn{6}{c|}{\scriptsize{classification accuracy}}\\
{atomic-level}&\multicolumn{6}{c|}{{recognition accuracy}}\\
\cline{2-7}{error rate}&{IHMM}&{DBN}&{ITBN}&{IBGN-C}&{IBGN-F}&{IBGN}\\
%\hline
%\multicolumn{7}{|c|}{No Error}\\
%\hline
%0&&0.94&&0.88&&\\
\hline
\multicolumn{7}{|c|}{{under synthetic atomic-level misdetection errors}}\\
\hline
0.1&0.31&0.73&0.79&0.92&0.88&\textbf{0.94}\\
0.2&0.29&0.67&0.74&\textbf{0.87}&\textbf{0.87}&\textbf{0.87}\\
0.3&0.22&0.65&0.71&0.83&\textbf{0.84}&\textbf{0.84}\\
\hline
\multicolumn{7}{|c|}{{under synthetic atomic-level duration detection errors}}\\
\hline
0.1&0.16&0.45&0.76&0.83&0.88&\textbf{0.89}\\
0.2&0.16&0.45&0.72&0.83&\textbf{0.85}&\textbf{0.85}\\
0.3&0.16&0.45&0.69&0.82&\textbf{0.83}&\textbf{0.83}\\
\hline
\multicolumn{7}{|c|}{{under real atomic-level detected errors}}\\
\hline
0.165 ({kNN})&0.66&0.62&0.54&\textbf{0.91}&0.79&\textbf{0.91}\\
\hline
0.242 ({SVM})&0.58&0.54&0.46&0.83&0.71&\textbf{0.85}\\
%\hline
%0.243 (C4.5)&0.33&0.13&0.54&0.29&\textbf{0.83}&0.71\\
\hline
0.315 ({DT})&0.16&0.04&0.38&\textbf{0.71}&0.54&\textbf{0.71}\\
\hline
\end{tabular}
\end{threeparttable}
\end{center}
\end{table}
\end{comment}

\subsection{Our Complex Hand Activity Dataset}
%To our best knowledge, the three above mentioned datasets are so far the only ones publicly available for the field of complex activity recognition. In particular, the number of instances are on the order of at most hundreds, which is relatively small considering the big-data era we are embracing with.
%To this end,
\setcounter{paragraph}{0}
\paragraph{Data Collection}
We propose a new complex activity dataset on depth camera-based complex hand activities on performing American Sign Language (ASL).
It is an ongoing effort, and at the moment it contains 3,480 annotated instances, which is already about 5-fold larger than existing ones.
%Once ready we plan to share the dataset and related tools to the everyone in the community.
As illustrated in Fig.~\ref{fig:experiment-example}, complex activities in our dataset are defined as selected ASL hand-actions. There are 20 atomic actions, which are defined as the states of individual fingers, either straight or bent. It is important to realize that in a complex activity, there could be multiple occurrences of the same atomic action, as is also exemplified in Fig.~\ref{fig:example-f} where an action A2 appears twice in the sub-network.
\begin{figure}
\centering
\subfigure[\emph{Ketchup} (Case I)]{
\includegraphics[width=0.46\columnwidth]{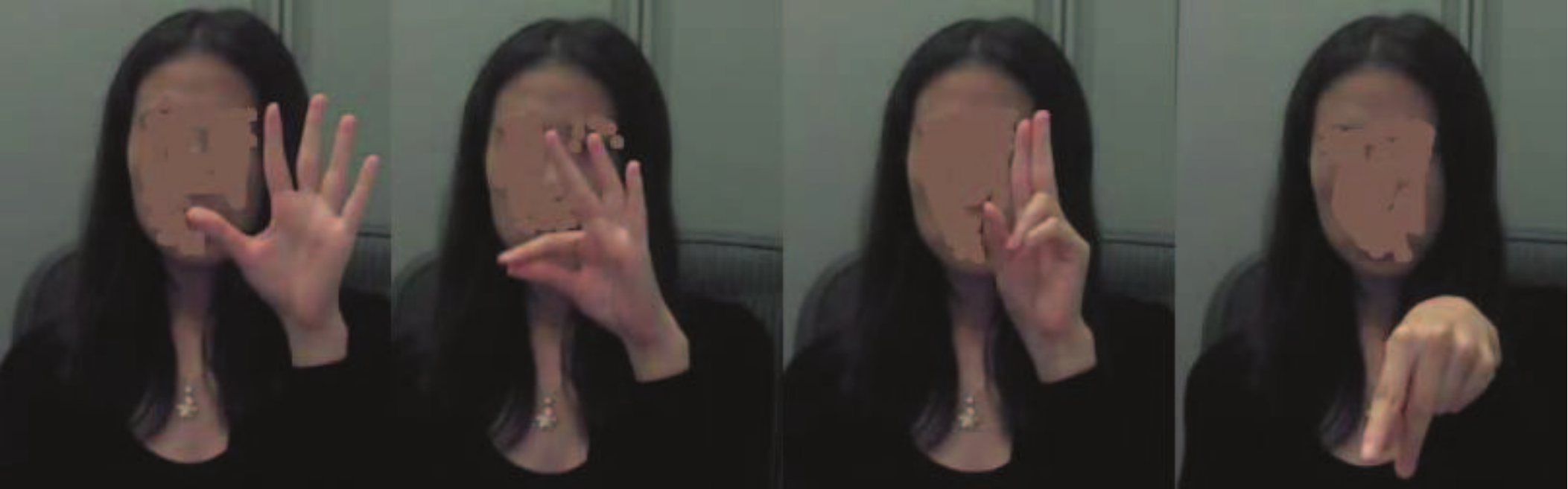}
\label{fig:example-a}
}
\subfigure[\emph{Ketchup} (Case II)]{
\includegraphics[width=0.46\columnwidth]{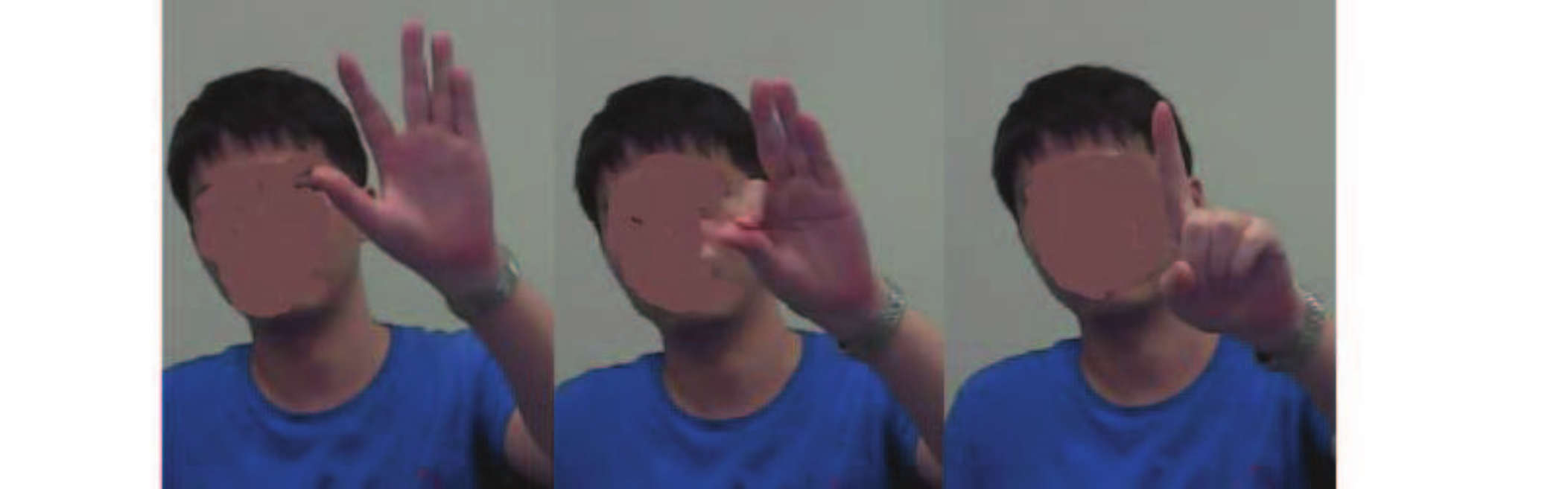}
\label{fig:example-b}
}
\subfigure[Intervals (Case I)]{
\includegraphics[width=0.46\columnwidth]{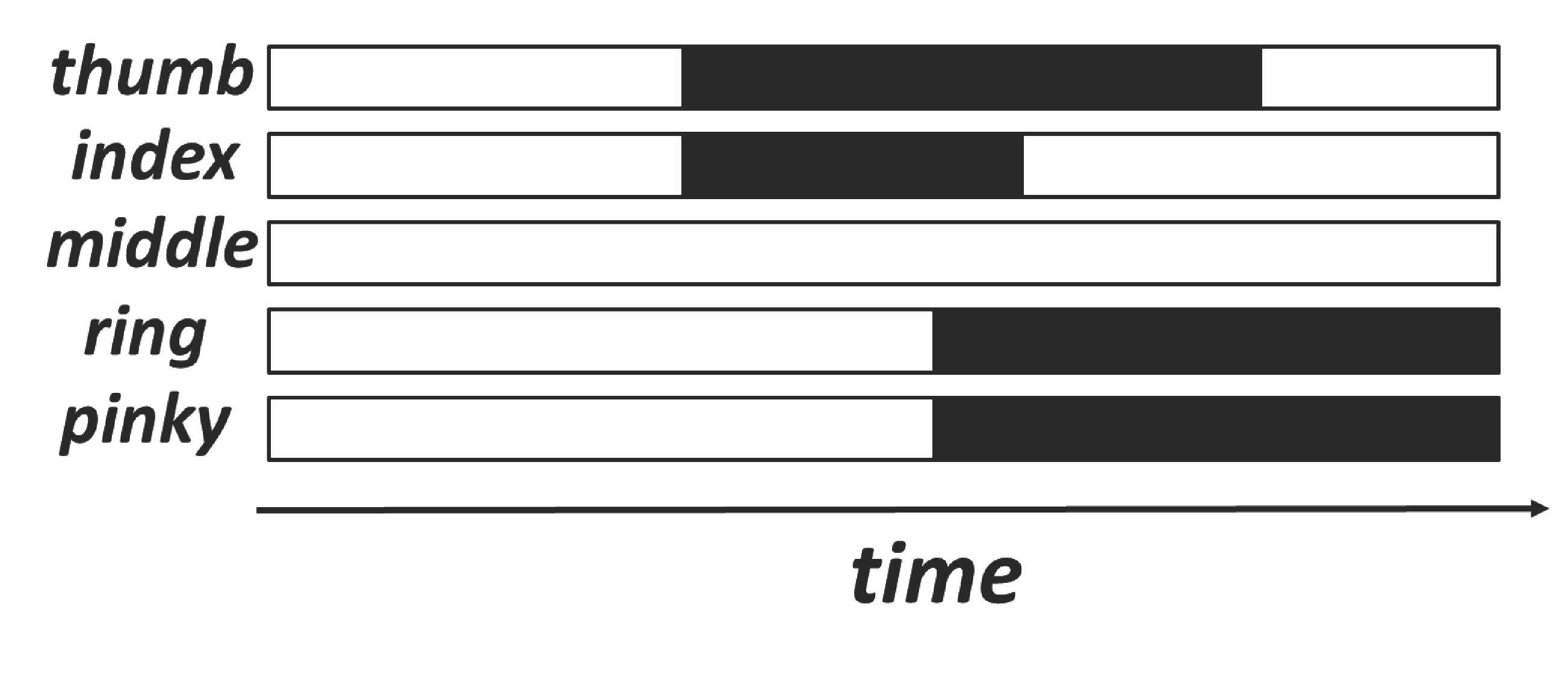}
\label{fig:example-c}
}
\subfigure[Intervals (Case II)]{
\includegraphics[width=0.46\columnwidth]{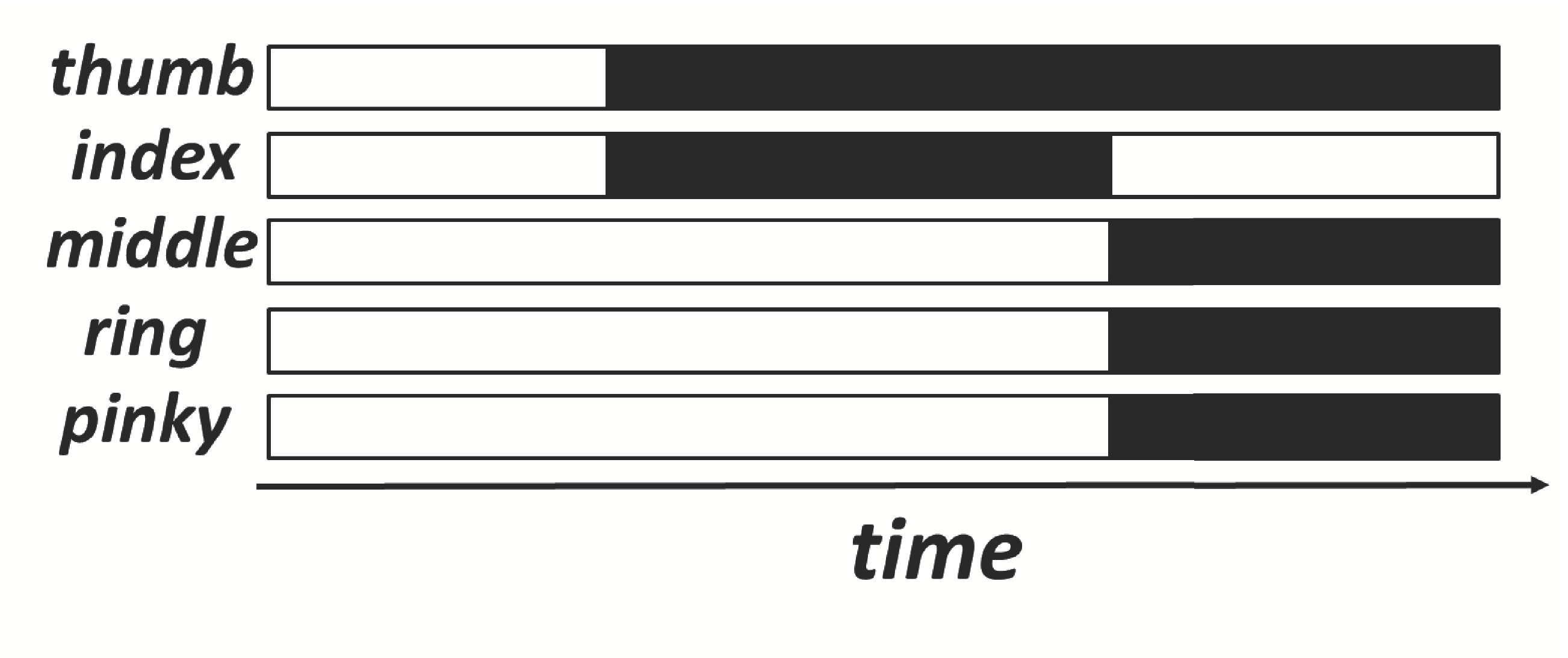}
\label{fig:example-d}
}
\\
\subfigure[A fraction of the interval-based network (Case I)]{
\includegraphics[width=0.46\columnwidth]{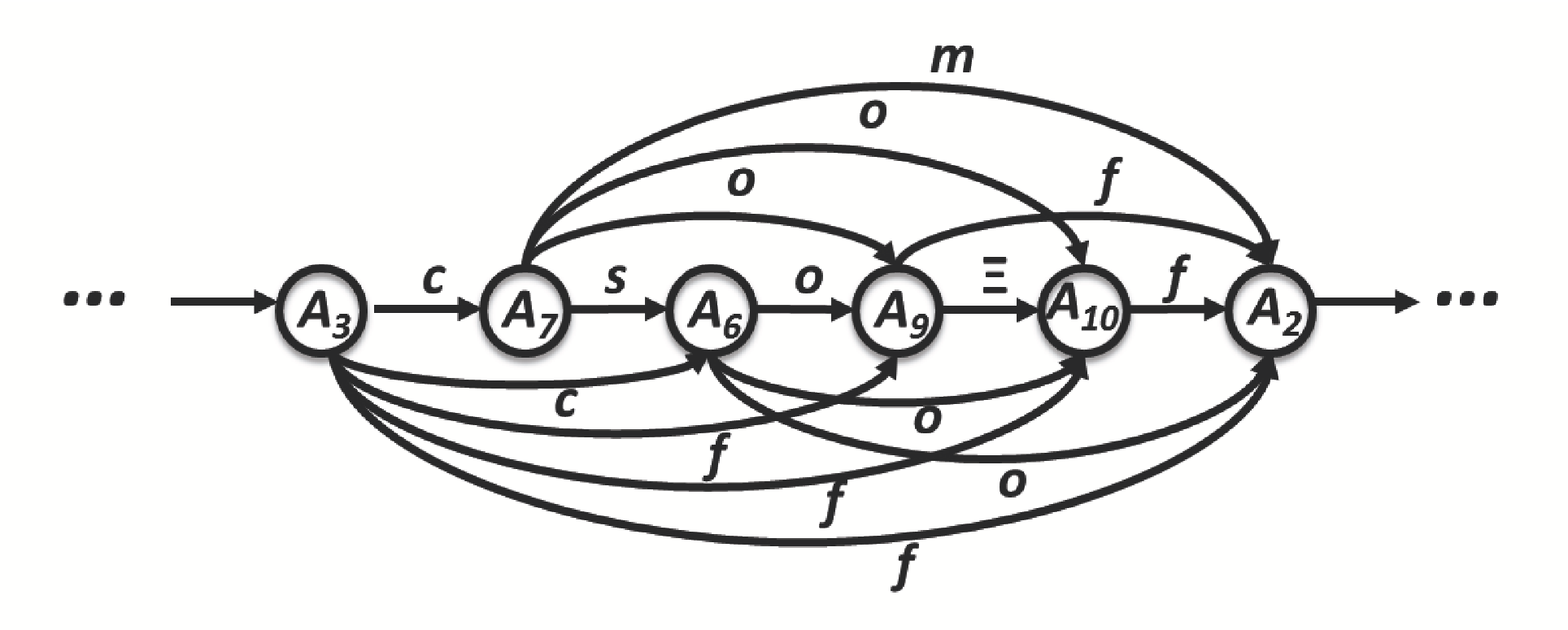}
\label{fig:example-e}
}
\subfigure[A fraction of the interval-based network(Case II)]{
\includegraphics[width=0.46\columnwidth]{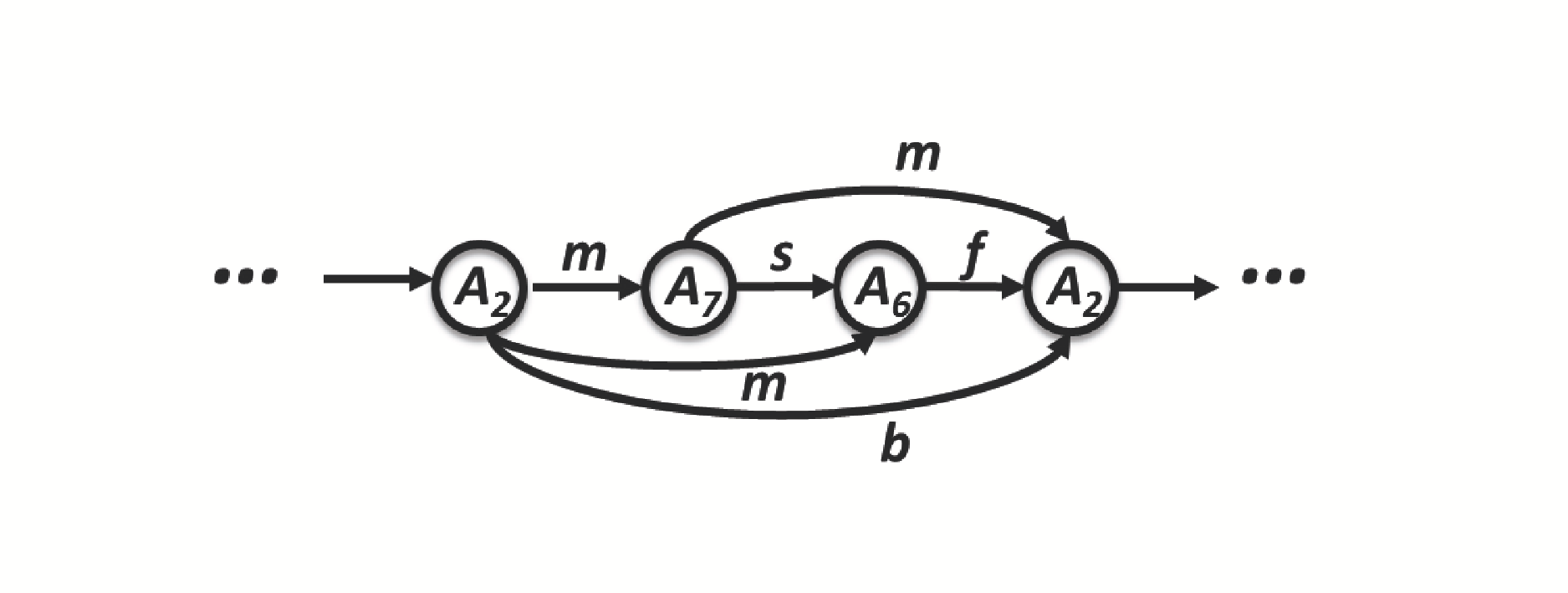}
\label{fig:example-f}
}
\caption{{Two instances of the complex activity ASL word \emph{ketchup}. $A_1$-$A_5$ refer to the straightening state of thumb, index, middle, ring and pinky fingers, respectively (white bars), while $A_6$-$A_{10}$ refer to the bending state of these fingers (black bars).}}
\label{fig:experiment-example}
\end{figure}

%In a complex activity, a transition is defined as the change from one static pose to another.
%The ASL actions are selected in such a way that more than one complex activity could correspond to the same noun in order to test the robustness of the proposed model.
Sixteen subjects participate in the data collection, with various factors being taken into consideration to add to the data diversities. Subjects of different genders, ages groups, races are present in the dataset. The male to female ratio is 12 to 4, races ratio is 14 to 2 and participants' ages span from around 15 to 40.
%For each subject, the RGB, confidence, and depth images are recorded while performing designated complex activities. In total, the dataset contains 24 different types of complex activities corresponding to 19 American sign language vocabularies, which are listed in Tab.~\ref{transition-table}. Each complex activity will be performed by all the subjects for 145 times in total, and there are 3480 complex activity instances in the dataset. On average, there are 1.875 transitions in each complex activity, and the minimum and maximum transition numbers are 1 and 4 respectively. The detailed parameters of the dataset are summarized in Table~\ref{dataset-parameters}. For each recorded complex activity, the RGB images are converted to videos, and both the transitions and atomic actions are annotated.
For each subject, the depth image sequences are recorded with a front-mount SoftKinetic camera while performing designated complex activities in office settings, with a frame-rate of 25 FPS, image size of 320$\times$240, and hand-camera distance of around 0.6-1m.
%To simplify the matter, only single right hands are considered in these images.
In total, the dataset contains 19 ASL hand-action complex activities, with each having 145-290 instances collected among all subjects. Each of the instance is comprised of 5-17 atomic action intervals. The 19 ASL hand-actions are \emph{air, alphabets, bank, bus, gallon, high school, how much, ketchup, lab, leg, lady, quiz, refrigerator, several, sink, stepmother, teaspoon, throw, xray}.
%In this experiment, three subjects are randomly chosen as test subjects, while others are considered as training set.
%

\paragraph{Atomic Hand Action Detection}
To detect atomic-level hand-actions, we make use of the existing hand pose estimation system~\cite{xu2016estimate} with a postprocessing step to map joint location prediction outputs to the bent/straight states of fingers.
To evaluate performance of the interval-level atomic action detection results, we follow the common practice and use the intersection-over-union of intervals with a 50$\%$ threshold to identify a hit/false alarm/missing, respectively. Finally F1 score is used based on the obtained precision and recall values. Note here the finger bent states are considered as foreground intervals.

\paragraph{Robustness Tests under Atomic-Level Errors}
We first evaluate the performance on simulated synthetic misdetection errors and duration-detection errors. From Fig.~\ref{tab-system-comparison-with-synthetic-errors}, we observe that overall IBGN is notably more robust than other approaches, meanwhile IHMM consistently produces the worst results. Our model is relatively robust in the presence of atomic-level errors.
%especially when the error rate is under 30\%.
\begin{figure}
\centering
\subfigure[with misdetection errors]{
\includegraphics[width=0.46\columnwidth]{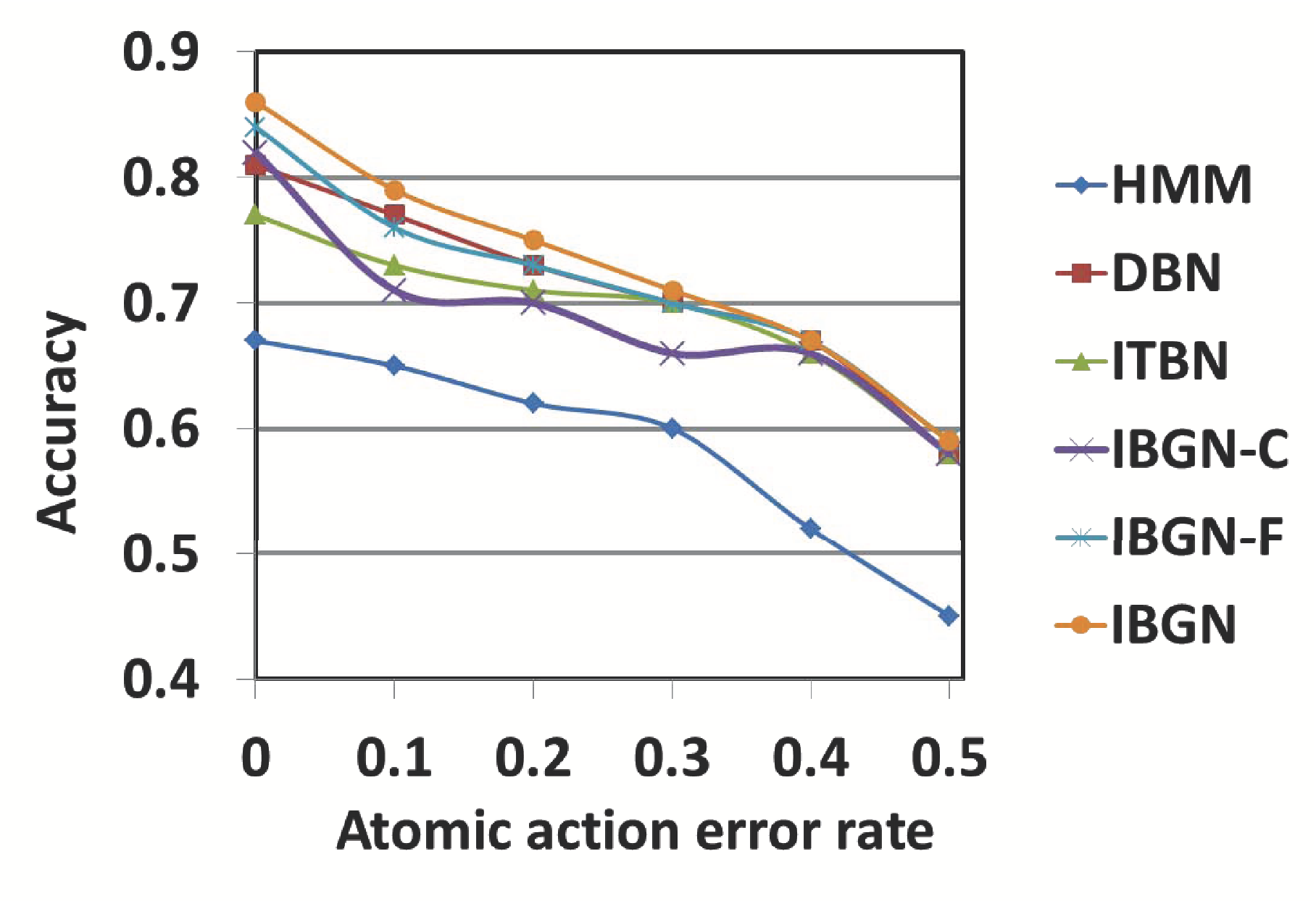}
\label{fig:example-a}
}
%\hspace{-2mm}
\subfigure[with duration detection errors]{
\includegraphics[width=0.46\columnwidth]{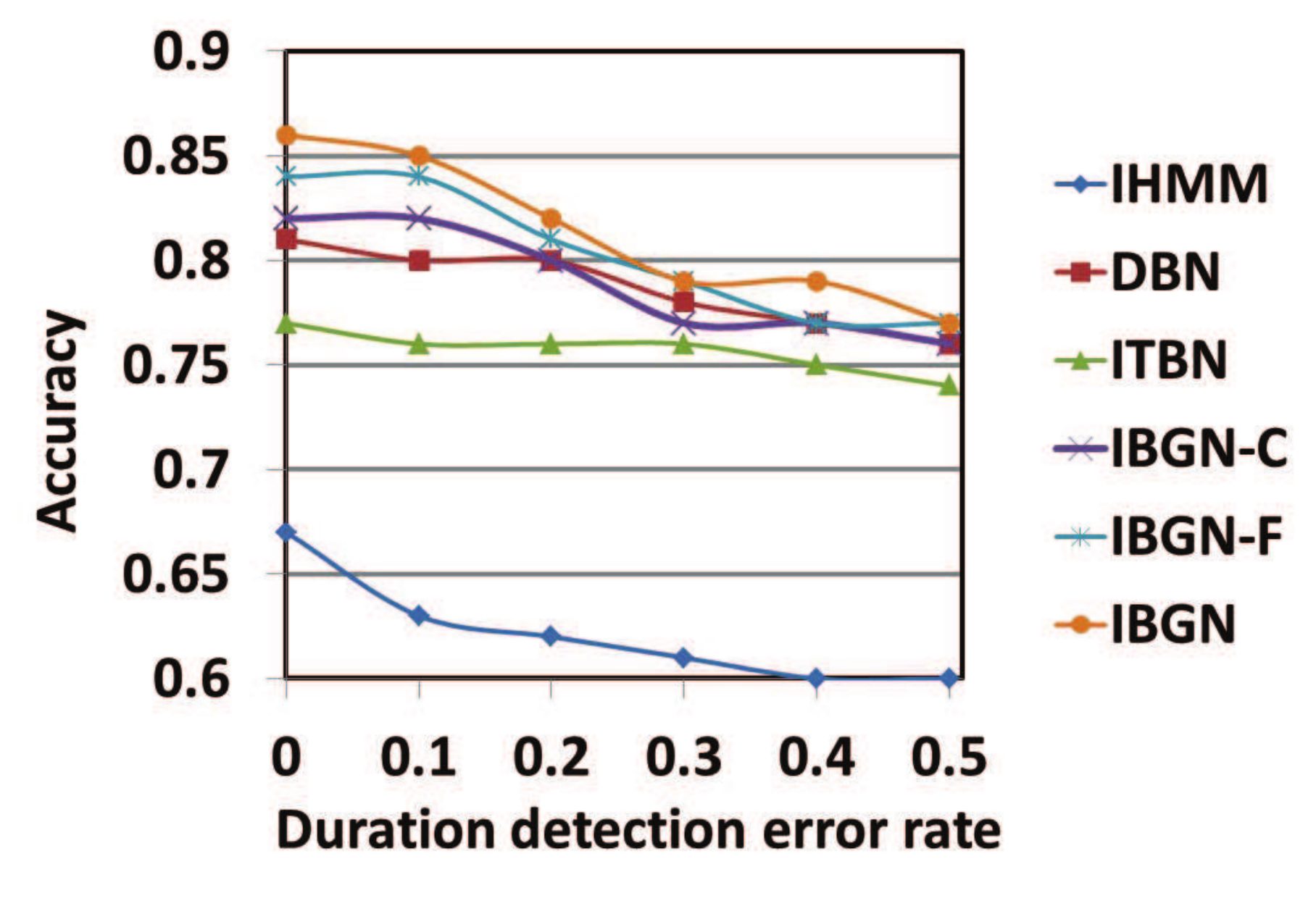}
\label{fig:example-b}
}
\caption{{Performance changes vs. perturbation of atomic-action-level errors.}}
\label{tab-system-comparison-with-synthetic-errors}
\end{figure}
\begin{comment}
\begin{table}[!htbp]
\caption{Accuracies under synthetic errors of atomic action recognition.}
\label{tab-system-comparison-with-synthetic-errors}
\center
\scriptsize
\scalebox{0.75}{
\begin{tabular}{|c||cccccc|}
%\begin{tabular}{clcc}
\hline
%\multirow{2}{*}{\scriptsize{atomic actions error rate}}&\multicolumn{6}{c|}{\scriptsize{classification accuracy}}\\
\scriptsize{atomic action}&\multicolumn{6}{c|}{\scriptsize{classification accuracy}}\\
\cline{2-7}\scriptsize{error rate}&\scriptsize{HMM}&\scriptsize{DBN}&\scriptsize{ITBN}&\scriptsize{IBGN-C}&\scriptsize{IBGN-F}&\scriptsize{IBGN}\\
%\hline
%\multicolumn{7}{|c|}{No Error}\\
%\hline
%0&&0.94&&0.88&&\\
\hline
\multicolumn{7}{|c|}{\scriptsize{under no errors}}\\
\hline
0 &0.67&0.83&0.77&0.82&0.84&\textbf{0.85}\\
\hline
\multicolumn{7}{|c|}{\scriptsize{under synthetic misdetection errors}}\\
\hline
0.1&0.65&&0.73&0.71&0.76&\textbf{0.78}\\
0.2&0.62&&0.71&0.70&0.73&\textbf{0.74}\\
0.3&0.60&&0.66&0.66&\textbf{0.67}&\textbf{0.67}\\
0.4&0.52&&0.62&0.66&\textbf{0.70}&\textbf{0.70}\\
0.5&0.45&&0.52&0.58&\textbf{0.59}&0.58\\
\hline
\multicolumn{7}{|c|}{\scriptsize{under synthetic time detection errors}}\\
\hline
0.1&0.63&&0.77&0.82&0.84&\textbf{0.85}\\
0.2&0.62&&0.77&0.80&0.81&\textbf{0.82}\\
0.3&0.61&&0.77&0.77&\textbf{0.79}&\textbf{0.79}\\
0.4&0.60&&0.75&\textbf{0.77}&\textbf{0.77}&\textbf{0.77}\\
0.5&0.60&&0.74&0.76&\textbf{0.77}&\textbf{0.77}\\
\hline
\end{tabular}
}
\vspace{-4mm}
\end{table}
\end{comment}

Now we are ready to show the performance of our model when working with our atomic-level predictor as mentioned previously. Overall our atomic-level predictor achieves F1 score of 0.724.
%Overall our low-level predictor achieves F1 score of xxx. More concretely its performance over categories of complex activities, and over atomic actions are presented in Fig.xxx and xxx. [Also analysis here!!!!!].
%
Table~\ref{tab-experiment-II-comparison-precision} summarizes the final accuracy comparisons of the six competing approaches based on our atomic-level predictor vs. the atomic-level ground-truth labels on our hand-action dataset. It is not surprising that in both scenarios IBGN again significantly outperforms the rest approaches.
It is worth noting that taking into account the challenging task of atomic-level hand pose estimation on its own, the gap in performances of $0.58$ vs. $0.86$ on predicting over 19 complex activity categories is reasonably, which is also certainly one thing we should improve over in the future.
\begin{table}[!htbp]
\caption{{Complex activity accuracy comparisons on our hand-action dataset.}}
\label{tab-experiment-II-comparison-precision}
\begin{center}
\scriptsize
\scalebox{1}{
\begin{tabular}{|c|c|c|c|c|c|}
%\begin{tabular}{clcc}
\hline
\multicolumn{1}{|c|}{IHMM}&\multicolumn{1}{c|}{DBN}&\multicolumn{1}{c|}{ITBN}&\multicolumn{1}{c|}{IBGN-C}&\multicolumn{1}{c|}{IBGN-F}&\multicolumn{1}{c|}{IBGN}\\
\hline
\multicolumn{6}{|c|}{with real atomic-level prediction}\\
\hline
0.43&0.51&0.54&0.49&0.55&\textbf{0.58}\\
\hline
\multicolumn{6}{|c|}{with atomic-level ground-truth (ideal situation)}\\
\hline
0.67&0.81&0.77&0.82&0.84&\textbf{0.86}\\
\hline
\end{tabular}
}
\end{center}
\end{table}
Fig.~\ref{fig-experiment-real-confusionmatri} presents the confusion matrix of IBGN working with our atomic-level predictor. We observe that our system is able to recognize the ASL words such as \emph{xray}, \emph{alphabets} and \emph{quiz} very well. At the same time, several ASL words turn to be difficult to deal with, with accuracy under 50\%. This may mainly due to the relatively low accuracy of the atomic-level predictor we are using on the particular atomic actions.
%At the same time, among those ASL words with the accuracy under 50\%, many are wrongly recognized as two other words, \emph{ketchup} and \emph{xray}. This is because of the low accuracy on recognizing atomic finger actions for these two words, which we will continue to improve in our future work.
\begin{figure}
\centering
\includegraphics[width=1.0\columnwidth]{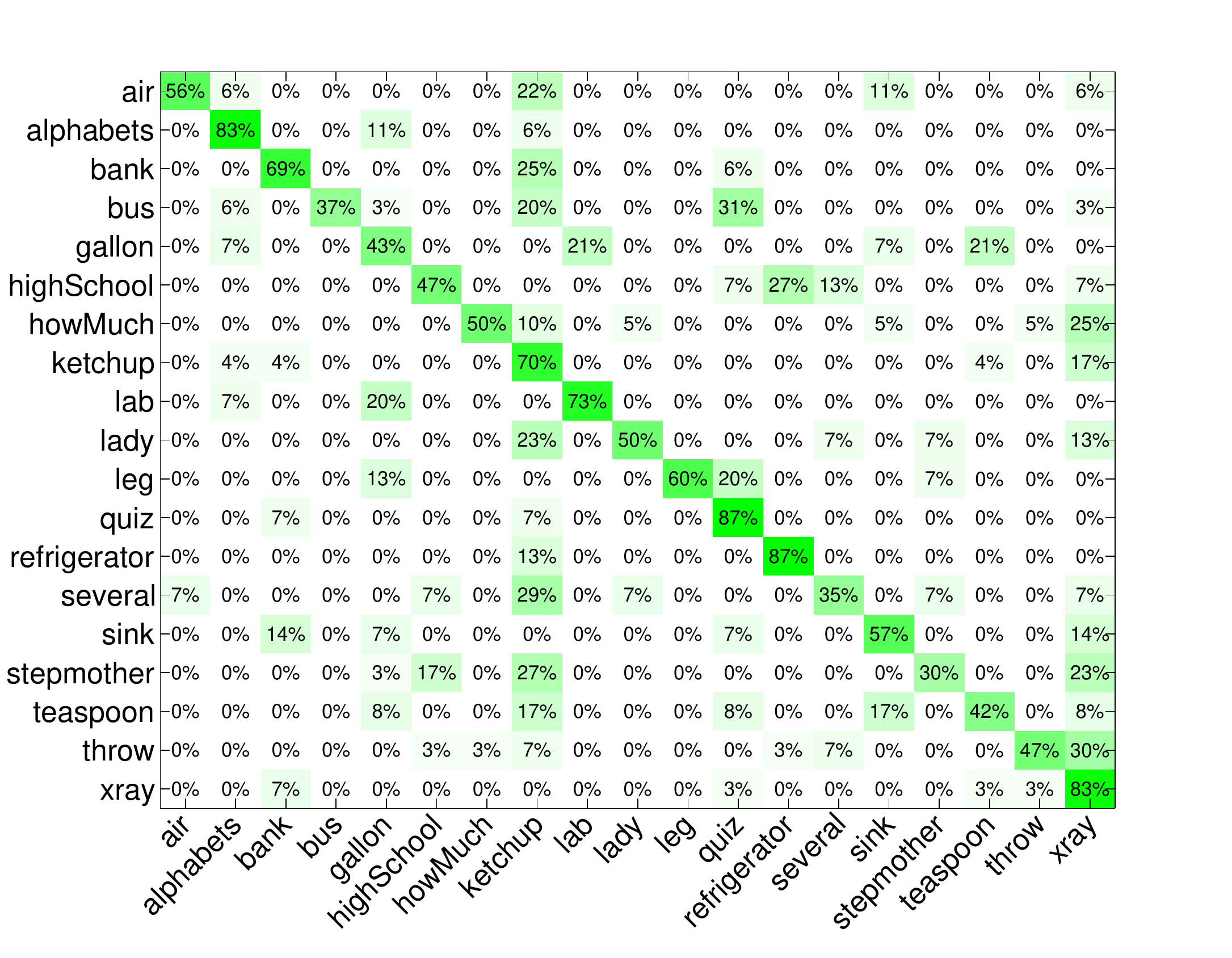}
\caption{{Confusion matrix of IBGN on our hand-action dataset.}}
\label{fig-experiment-real-confusionmatri}
\end{figure}

\section{Conclusion}
We present an interval-based Bayesian generative network approach to account for the latent structures of complex activities by constructing probabilistic interval-based networks with temporal dependencies in complex activity recognition.
In particular, the Bayesian framework and the novel application of Chinese restaurant process (\emph{CRP}) of our IBGN model enable us to explicitly capture inherit structural variability in each of the complex activities.
%Our approach is able to efficiently describe and propagate rich temporal dependencies as well as uncertainties among atomic actions.
%Moreover, according to our empirical studies, it is more robust to errors introduced from atomic action level than existing methods for complex activity recognition.
In addition, we make publicly available a new complex hand activity dataset dedicated to the purpose of complex activity recognition,
which contains around an-order-of-magnitude larger number of annotated instances.
Experimental results suggest our proposed model outperforms existing state-of-the-arts by a large margin in a range of well-known testbeds as well as our new dataset. It is also shown that our approach is rather robust to the errors introduced by the low-level atomic action predictions from raw signals.
As part of future work, we are considering relaxing the assumption that the IBGN models share the same structure for representing multiple instances of the same complex activity, and will instead learn more flexible structures for each class of complex activities by introducing latent structure variables that decide whether a link should exist in an instance. Also, we will continue the finalization of our hand-action dataset, improving the atomic-level atomic action prediction method, as well as attempting toward the establishment of standardized comparisons on exiting systems.

\section{Acknowledgments}
This research was supported in part by grant CQU903005203326 from the Fundamental Research Funds for the Central Universities in China, grants R-252-000-473-133 and R-252-000-473-750 from the National University of Singapore, and A*STAR JCO grants 15302FG149 and 1431AFG120.

%\appendix
%\renewcommand{\appendixname}{Appendix~\Alph{section}}
\begin{appendices}
\section{Proofs of Theorem~\ref{thm:consistency}}
\label{appendix:consistency}
We first prove that the composition operation $\circ$ on the interval relation union set $\mathbb{S}$ satisfies the associative law. The set of all the $127$ unions is denoted by $\mathbb{S}$.
\begin{definition}
\text{(\textbf{Composition Product of Two Relation Sets})}\\
The composition product of two sets of interval relations $R,R'\in\mathbb{S}$, i.e. $R=\{r_{1},\ldots,r_{\mid R\mid}\}$ and $R'=\{r'_{1},\ldots,r'_{\mid R'\mid}\}$, where any $r_{i},r'_{j}\in\mathbb{R}$, is defined as
%\begin{small}
%\small
$R\circ R'=\bigcup_{i,j}{(r_{i}\circ r'_{j})}$.
%\end{small}
\end{definition}

\begin{lemma}
\label{thm:associative}
\text{(\textbf{Associative Law on Composition})}\\
\begin{equation}
\notag
\small
\begin{split}
(\mathbf{R}_{x}\circ\mathbf{R}_{y})\circ\mathbf{R}_{z}=\mathbf{R}_{x}\circ(\mathbf{R}_{y}\circ\mathbf{R}_{z}),&~~~~\text{for any $\mathbf{R}_{x},\mathbf{R}_{y},\mathbf{R}_{z}\in\mathbb{S}$.}
\end{split}
\end{equation}
\end{lemma}
\begin{proof}
Let
\begin{small}
$\mathbf{R}_{x}=\{r_{x,1},\ldots,r_{x,X}\}$, $\mathbf{R}_{y}=\{r_{y,1},\ldots,r_{y,Y}\}$
\end{small}
and
\begin{small}
$\mathbf{R}_{z}=\{r_{z,1},\ldots,r_{z,Z}\}$
\end{small}
, where any $r_{i,j}\in\mathbb{R}$. Then,
%\begin{equation}
%\notag
%\small
%\begin{split}
\begin{small}
$(\mathbf{R}_{x}\circ\mathbf{R}_{y})\circ\mathbf{R}_{z}=(\bigcup_{i,j}{(r_{x,i}\circ r_{y,j})})\circ\mathbf{R}_{z}
=\bigcup_{k}((\bigcup_{i,j}{(r_{x,i}\circ r_{y,j})})\circ r_{z,k})
=\bigcup_{i,j,k}({(r_{x,i}\circ r_{y,j})\circ r_{z,k}})$
\end{small}
%\end{split}
%\end{equation}
and
%\begin{equation}
%\notag
%\begin{split}
\begin{small}
$\mathbf{R}_{x}\circ(\mathbf{R}_{y}\circ\mathbf{R}_{z})=\mathbf{R}_{x}\circ(\bigcup_{j,k}{(r_{y,j}\circ r_{z,k})})
=\bigcup_{i}(r_{x,i}\circ(\bigcup_{j,k}{(r_{y,j}\circ r_{z,k})}))
=\bigcup_{i,j,k}({r_{x,i}\circ(r_{y,j}\circ r_{z,k})})$
\end{small}
%\end{split}
%\end{equation}
Since the composition on the seven interval relation set satisfies the associative law that is both left and right associative, the associative law also holds on $\mathbb{R}$, which is closed under the composition operation, i.e.
%\begin{equation}
%\notag
%\begin{split}
\begin{small}
$(\mathbf{r}_{x}\circ\mathbf{r}_{y})\circ\mathbf{r}_{z}=\mathbf{r}_{x}\circ(\mathbf{r}_{y}\circ\mathbf{r}_{z})$, for any $\mathbf{r}_{x},\mathbf{r}_{y},\mathbf{r}_{z}\in\mathbb{R}$.
\end{small}
%\end{split}
%\end{equation}
So, we have
\begin{small}
$(\mathbf{R}_{x}\circ\mathbf{R}_{y})\circ\mathbf{R}_{z}=\mathbf{R}_{x}\circ(\mathbf{R}_{y}\circ\mathbf{R}_{z})$
\end{small}.
\end{proof}

\begin{lemma}
\label{thm:path-consistency}
\text{(\textbf{Path Consistency})}\\
Given an IBGN $G_d$, if $\mathbf{x}_{i,k}\in\mathbf{x}_{i,j}\circ\mathbf{x}_{j,k}$ for any $1\leq i<j<k\leq \lvert V_d \rvert$, then for any path $\mathbf{P}=v_{i}\rightarrow v_{i'}\rightarrow v_{i''}\rightarrow\ldots\rightarrow v_{k'}\rightarrow v_{k}$ in $G_d$, $\mathbf{x}_{i,k}\in\mathbf{x}_{i,i'}\circ\mathbf{x}_{i',i''}\circ\ldots\circ\mathbf{x}_{k',k}$.
\end{lemma}
\begin{proof}
For any path $\mathbf{P}=v_{i}\rightarrow v_{i'}\rightarrow v_{i''}\rightarrow\ldots\rightarrow v_{k'}\rightarrow v_{k}$ in $G_{d}$, we have
%\begin{displaymath}
%\begin{small}
$\mathbf{x}_{i,k}\in\mathbf{x}_{i,i'}\circ\mathbf{x}_{i',k}$
%\end{small}
%\end{displaymath}
and
%\begin{displaymath}
$\mathbf{x}_{i',k}\in\mathbf{x}_{i',i''}\circ\mathbf{x}_{i'',k}$,
%\end{displaymath}
then
%\begin{displaymath}
$\mathbf{x}_{i,k}\in\mathbf{x}_{i,i'}\circ(\mathbf{x}_{i',i''}\circ\mathbf{x}_{i'',k})\in\mathbf{x}_{i,i'}\circ\mathbf{x}_{i',i''}\circ\mathbf{x}_{i'',k}$.
%\end{displaymath}
Iteratively, we finally have
%\begin{displaymath}
$\mathbf{x}_{i,k}\in\mathbf{x}_{i,i'}\circ\mathbf{x}_{i',i''}\circ\ldots\circ\mathbf{x}_{k',k}$.
%\end{displaymath}
\end{proof}

\begin{proof}
~\\
\textbf{(Consistency)}

Lemma~\ref{thm:path-consistency} indicates that if any $\triangle ijk$ in the network satisfies the transitivity properties, then any path in the network also satisfies the transitivity properties. Hence, the entire network generated through our construction is consistent.
~\\
\textbf{(Completeness)}
~\\
(by contradiction)

Suppose there exists one interval relations $\mathbf{x}_{i,n}$ that cannot be generated through our construction, i.e. $\mathbf{x}_{i,n}\notin\bigcap_{j=i+1}^{n-1}{(\mathbf{x}_{i,j}\circ\mathbf{x}_{j,n})}$, which means there exists at least one $j^*$ ($1\leq j^*\leq n-1$) that $\mathbf{x}_{i,n}\notin\mathbf{}_{i,j^*}\circ\mathbf{x}_{j^*,n}$. This contradicts the fact that the network is consistent. Besides, any possible relation in $\mathbb{R}$ can be generated between two neighboring nodes. That is to say, any possible triangle that is consistent can be generated by our model.
\end{proof}

\section{Proof of $\text{BIC-Score}(\mathbf{G}:\mathbb{D})=\text{BIC-Score}(\mathbf{G}':\mathbb{D})+\lambda$}
\label{appendix:bic}
\begin{proof}

The corresponding variable dependence in the IBGN model $\mathbf{G}$ is shown as follows:\\
%\vspace{1cm}
\begin{displaymath}
\scalebox{0.7}{
\xymatrix@C=0.3cm{
  \mathbf{t}_1 \ar@{-->}@[blue][d] \ar@{-->}@[blue][r] \ar@/^0.4pc/@{-->}@[blue][rr] \ar@/^0.8pc/@{-->}@[blue][rrr] \ar@/^1.2pc/@{-->}@[blue][rrrrr] \ar@/^1.6pc/@{-->}@[blue][rrrrrr]
  & \mathbf{t}_2 \ar@{-->}@[blue][d] \ar@{-->}@[blue][r] \ar@/^0.4pc/@{-->}@[blue][rr] \ar@/^0.8pc/@{-->}@[blue][rrrr] \ar@/^1.2pc/@{-->}@[blue][rrrrr]
  & \ldots \ar@{-->}@[blue][r]
  & \mathbf{t}_{k} \ar@{-->}@[blue][d] \ar@{-->}@[blue][r] \ar@/^0.4pc/@{-->}@[blue][rr] \ar@/^0.8pc/@{-->}@[blue][rrr]
  & \ldots \ar@{-->}@[blue][r]
  & \mathbf{t}_{k^*-1} \ar@{-->}@[blue][d] \ar@{-->}@[blue][r]
  & \mathbf{t}_{k^*} \ar@{-->}@[blue][d]\\
  \mathbf{a}_1 \ar[d] \ar[rrd] \ar[rrrrd] \ar[rrrrrd]
  & \mathbf{a}_2 \ar[ld]
  & \ldots\ldots
  & \mathbf{a}_{k} \ar[ld] \ar[rrrrd] \ar[rrrrrd]
  & \ldots\ldots
  & \mathbf{a}_{k^*-1} \ar[ld] \ar[rrd] \ar[rrrrrd]
  & \mathbf{a}_{k^*} \ar[ld] \ar[rrd] \ar[rrrrd]\\
  \mathbf{r}_{1,2}
  & \ldots
  & \mathbf{r}_{1,k}
  & \ldots
  & \mathbf{r}_{1,k^*-1}
  & \mathbf{r}_{1,k^*}
  & \ldots
  & \mathbf{r}_{k,k^*-1}
  & \mathbf{r}_{k,k^*}
  & \ldots
  & \mathbf{r}_{k^*-1,k^*}
  \\
  \mathbf{c}_{1,2} \ar@{-->}@[blue][u]
  & \ldots
  & \mathbf{c}_{1,k} \ar@{-->}@[blue][u]
  & \ldots
  & \mathbf{c}_{1,k^*-1} \ar@{-->}@[blue][u]
  & \mathbf{c}_{1,k^*} \ar@{-->}@[blue][u]
  & \ldots
  & \mathbf{c}_{k,k^*-1} \ar@{-->}@[blue][u]
  & \mathbf{c}_{k,k^*} \ar@{-->}@[blue][u]
  & \ldots
  & \mathbf{c}_{k^*-1,k^*} \ar@{-->}@[blue][u]
}}
\end{displaymath}
According to our definition of the generative process, the structure between $\mathbf{t}$ and $\mathbf{a}$ (marked as blue line), the structure among $\mathbf{t}$ and the structure between $\mathbf{c}$ and $\mathbf{r}$ (marked as blue dotted line) are fixed. We only need to learn the structure between $\mathbf{a}$ and $\mathbf{r}$.
%A Markovian property holds, which ensures that $P(\mathbf{r}_{i,j}|\mathbf{t},\mathbf{a},\mathbf{r}^-)=P(\mathbf{r}_{i,j}|\mathbf{a}_i,\mathbf{a}_j)$, where the superscript $-$ of $\mathbf{r}$ indicates the set of variables does not include $\mathbf{r}_{i,j}$.

In effect, the structure without $\mathbf{t}$ and $\mathbf{c}$ and their associated links (blue dotted lines) is equivalent to the Bayesian structure $\mathbf{G}'$ defined in Definition~\ref{def:ibgn-structure}, where $U_{\mathbf{G}'}=\{\mathbf{a}_{1},\mathbf{a}_{2},\ldots,\mathbf{a}_{k^*}\}$ and $W_{\mathbf{G}'}=\{\mathbf{r}_{1,2},\mathbf{r}_{1,3},\ldots,\mathbf{r}_{1,k^*},$ $\ldots, \mathbf{r}_{k^*-1,k^*}\}$ (constraint (1)). The number of categories of any atomic action variable $\mathbf{a}_{i}$ ($1\leq i\leq k^*$) is $M+1$, including the \emph{null} atomic action. Similarly, the number of categories of any relation variable $\mathbf{r}_{i,j}$ ($1\leq i<j\leq k^*$) is 8, including the \emph{null} relation (constraint (2)). Given a training dataset $\mathbb{D}$, for any instance $d\in\mathbb{D}$ and its corresponding network $G_d$, if $j>\lvert V_d\rvert$, then $\mathbf{a}_{j}=\text{\emph{null}}$ and $\mathbf{r}_{i,j}=\text{\emph{null}}$. Moreover, any atomic action variable $\mathbf{a}_{i}$ has no parent, and any relation variable $\mathbf{r}_{i,j}$ has either two parents or zero parents (constraint (3)). That is, a link $e_{i,j}$ exists in $\mathbf{G}$ if and only if both $\mathbf{a}_{i}$ and $\mathbf{a}_{j}$ are connected with $\mathbf{r}_{i,j}$ in $\mathbf{G}'$.

Formally, let
$V_{\mathbf{G}}=V_{\mathbf{t}}\bigcup V_{\mathbf{a}}\bigcup V_{\mathbf{r}}\bigcup V_{\mathbf{c}}$ where
$V_{\mathbf{t}}=\{\mathbf{t}_1,\ldots,\mathbf{t}_{k^*}\}$,
$V_{\mathbf{a}}=\{\mathbf{a}_1,\ldots,\mathbf{a}_{k^*}\}$,
$V_{\mathbf{r}}=\{\mathbf{r}_{1,2},\mathbf{r}_{1,3},\ldots,$ $\mathbf{r}_{1,k^*},\ldots,\mathbf{r}_{k^*-1,k^*}\}$
and $V_{\mathbf{c}}=\{\mathbf{r}_{1,2},\mathbf{c}_{1,3},\ldots,\mathbf{c}_{1,k^*},$ $\ldots,\mathbf{c}_{k^*-1,k^*}\}$. $\Pi_i$ denotes the parents of $v_i\in V_\mathbf{G}$ and $\widehat{\Pi_i}=\prod_{v_j\in \Pi_i}\widehat{v}_j$, where $\widehat{v}_j$ is the number of categories of $v_j$. In particular, we have
\begin{equation}
\nonumber
\footnotesize
\begin{split}
\widehat{v}_j=&\left\{
  \begin{array}{l l}
    \ell & \quad \text{if $v_j\in V_{\mathbf{t}}$},\\
    M+1 & \quad \text{if $v_j\in V_{\mathbf{a}}$},\\
    8 & \quad \text{if $v_j\in V_{\mathbf{r}}$},\\
    \mathrm{c}_j & \quad \text{if $v_j\in V_{\mathbf{c}}$}.\\
  \end{array}
  \right.
\end{split}
\end{equation}
where $\mathrm{c}_j$ is a constant. Suppose the corresponding variable $\mathbf{c}_{n',n}$ of the node $v_j\in V_{\mathbf{c}}$ equals $\lvert C_z\rvert$, we have
\begin{equation}
\nonumber
\footnotesize
\begin{split}
\mathrm{c}_j=&\left\{
  \begin{array}{l l}
    1 & \quad \text{if $1\leq z\leq 7$},\\
    3 & \quad \text{if $z=8$ or $9$},\\
    5 & \quad \text{if $z=10$},\\
    7 & \quad \text{if $z=11$}.\\
  \end{array}
  \right.
\end{split}
\end{equation}

Since the structure between $\mathbf{t}$ and $\mathbf{a}$, the structure among $\mathbf{t}$ and the structure between $\mathbf{c}$ and $\mathbf{r}$ are fixed, the number of parents of a node $v_i\in V_{\mathbf{t}}\bigcup V_{\mathbf{a}}\bigcup V_{\mathbf{c}}$ are fixed. Then, we have
\begin{equation}
\nonumber
\footnotesize
\begin{split}
\widehat{\Pi_i}=&\left\{
  \begin{array}{l l}
    \ell^{i-1} & \quad \text{if $v_i\in V_{\mathbf{t}}$},\\
    \ell & \quad \text{if $v_i\in V_{\mathbf{a}}$},\\
    1 & \quad \text{if $v_i\in V_{\mathbf{c}}$}.
  \end{array}
  \right.
\end{split}
\end{equation}
Also, let $\Phi$ be the entire vector of parameters such that $\forall_{ijk}:\phi_{ijk}=P(\mathrm{v}_{ik}\mid \pi_{ij})$, where $\mathrm{v}_{ik}$ and $\pi_{ij}$ denotes that $v_i$ is assigned with the $k$-th element and its parent is assigned with the $j$-th element in $\widehat{\Pi}_i$, respectively; $n_{ijk}$ indicates how many instances of $\mathbb{D}$ contain both $\mathrm{v}_{ik}$ and $\pi_{ij}$. Here, we split the parameter vector $\Phi$ into three parts: $\Phi^{(1)}$, $\Phi^{(2)}$, $\Phi^{(0)}$ and $\Phi^{(3)}$ are the parameter vectors associated with $V_{\mathbf{t}}$, $V_{\mathbf{a}}$, $V_{\mathbf{r}}$ and $V_{\mathbf{c}}$, respectively. Given a training dataset $\mathbb{D}_l$, we have
\begin{equation}
\nonumber
\scriptsize
\begin{split}
&\text{BIC-Score}(\mathbf{G}:\mathbb{D}_l)=\mathop{\max}\limits_{\Phi}{L_{\mathbf{G},\mathbb{D}_l}(\Phi)}-\frac{\log\lvert\mathbb{D}_l\rvert}{2}\cdot\lvert\Phi\rvert\\
&=\mathop{\max}\limits_{\Phi}{\log\mathop{\prod}\limits_{i=1}^{\lvert V_{\mathbf{G}}\rvert}\mathop{\prod}\limits_{j=1}^{\widehat{\Pi_i}}\mathop{\prod}\limits_{k=1}^{\widehat{v_i}}}{\phi_{ijk}^{n_{ijk}}}
-\frac{\log\lvert\mathbb{D}_l\rvert}{2}\cdot\mathop{\sum}\limits_{i=1}^{\lvert V_{\mathbf{G}}\rvert}{\widehat{\Pi_{i}}\cdot(\widehat{v_i}-1)}\\
&=\mathop{\max}\limits_{\Phi}({\log\mathop{\prod}\limits_{v_i\in V_{\mathbf{t}}}\mathop{\prod}\limits_{j=1}^{\widehat{\Pi_i}}\mathop{\prod}\limits_{k=1}^{\widehat{v_i}}}{\phi_{ijk}^{n_{ijk}}}
+{\log\mathop{\prod}\limits_{v_i\in V_{\mathbf{a}}}\mathop{\prod}\limits_{j=1}^{\widehat{\Pi_i}}\mathop{\prod}\limits_{k=1}^{\widehat{v_i}}}{\phi_{ijk}^{n_{ijk}}}
+{\log\mathop{\prod}\limits_{v_i\in V_{\mathbf{r}}}\mathop{\prod}\limits_{j=1}^{\widehat{\Pi_i}}\mathop{\prod}\limits_{k=1}^{\widehat{v_i}}}{\phi_{ijk}^{n_{ijk}}})
+{\log\mathop{\prod}\limits_{v_i\in V_{\mathbf{c}}}\mathop{\prod}\limits_{j=1}^{\widehat{\Pi_i}}\mathop{\prod}\limits_{k=1}^{\widehat{v_i}}}{\phi_{ijk}^{n_{ijk}}})\\
&-\frac{\log\lvert\mathbb{D}_l\rvert}{2}\cdot(\mathop{\sum}\limits_{v_i\in V_{\mathbf{t}}}{\widehat{\Pi_{i}}\cdot(\widehat{v_i}-1)}
+\mathop{\sum}\limits_{v_i\in V_{\mathbf{a}}}{\widehat{\Pi_{i}}\cdot(\widehat{v_i}-1)}
+\mathop{\sum}\limits_{v_i\in V_{\mathbf{r}}}{\widehat{\Pi_{i}}\cdot(\widehat{v_i}-1)}
+\mathop{\sum}\limits_{v_i\in V_{\mathbf{c}}}{\widehat{\Pi_{i}}\cdot(\widehat{v_i}-1)})\\
&=\mathop{\max}\limits_{\Phi^{(1)}}{\log\mathop{\prod}\limits_{i=1}^{k^*}\mathop{\prod}\limits_{j=1}^{\ell^{i-1}}\mathop{\prod}\limits_{k=1}^{\ell}}{({\phi}_{ijk}^{(1)})^{n_{ijk}}}
-\frac{\log\lvert\mathbb{D}_l\rvert}{2}\cdot{k^*\cdot\ell^{i-1}\cdot(\ell-1)}\\
&+\mathop{\max}\limits_{\Phi^{(2)}}{\log\mathop{\prod}\limits_{i=1}^{k^*}\mathop{\prod}\limits_{j=1}^{\ell}\mathop{\prod}\limits_{k=1}^{M+1}}{(\phi_{ijk}^{(2)})^{n_{ijk}}}
-\frac{\log\lvert\mathbb{D}_l\rvert}{2}\cdot{k^*\cdot\ell\cdot M}\\
&+\mathop{\max}\limits_{\Phi^{(3)}}{\log\mathop{\prod}\limits_{i=1}^{k^*(k^*-1)/2}\mathop{\prod}\limits_{j=1}^{1}\mathop{\prod}\limits_{k=1}^{\mathrm{c}_i}}{(\phi_{ijk}^{(3)})^{n_{ijk}}}
-\frac{\log\lvert\mathbb{D}_l\rvert}{2}\cdot{\mathop{\sum}\limits_{i=1}^{k^*(k^*-1)/2}(\mathrm{c}_i-1)}\\
&+\mathop{\max}\limits_{\Phi^{(0)}}{\log\mathop{\prod}\limits_{v_i\in V_{\mathbf{r}}}\mathop{\prod}\limits_{j=1}^{1}\mathop{\prod}\limits_{k=1}^{8}}{(\phi_{ijk}^{(0)})^{n_{ijk}}}
-\frac{\log\lvert\mathbb{D}_l\rvert}{2}\cdot\frac{k^*(k^*-1)}{2}\cdot 7  \text{\quad\quad$\rhd$ for $v_i\in V_{\mathbf{r}}$ and $\Pi_i\subset V_{\mathbf{c}}$}\\
&+\mathop{\max}\limits_{\Phi^{(0)}}{\log\mathop{\prod}\limits_{v_i\in V_{\mathbf{r}}}\mathop{\prod}\limits_{j=1}^{\widehat{\Pi_i}}\mathop{\prod}\limits_{k=1}^{8}}{(\phi_{ijk}^{(0)})^{n_{ijk}}}
-\frac{\log\lvert\mathbb{D}_l\rvert}{2}\cdot\mathop{\sum}\limits_{v_i\in V_{\mathbf{r}}}{\widehat{\Pi_{i}}\cdot7} \text{\quad\quad\quad$\rhd$ for $v_i\in V_{\mathbf{r}}$ and $\Pi_i\subset V_{\mathbf{a}}$}.
\end{split}
\end{equation}

Similarly, in ${\mathbf{G}'}$, any node $v_i\in U_{\mathbf{G}'}$ has no parent, and thus $\widehat{\Pi_i}=1$.  We split the parameter vector $\Phi$ into two parts: $\Phi^{(4)}$ and $\Phi^{(0)}$ are the parameter vectors associated with $U_{\mathbf{G}'}$ and $W_{\mathbf{G}'}$, respectively. Then, we have
\begin{equation}
\nonumber
\scriptsize
\begin{split}
\text{BIC-Score}(\mathbf{G}':\mathbb{D}_l)&=\mathop{\max}\limits_{\Phi}{L_{\mathbf{G}',\mathbb{D}_l}(\Phi)}-\frac{\log\lvert\mathbb{D}_l\rvert}{2}\cdot\lvert\Phi\rvert\\
&=\mathop{\max}\limits_{\Phi}{\log\mathop{\prod}\limits_{i=1}^{\lvert V_{\mathbf{G}'}\rvert}\mathop{\prod}\limits_{j=1}^{\widehat{\Pi_i}}\mathop{\prod}\limits_{k=1}^{\widehat{v_i}}}{\phi_{ijk}^{n_{ijk}}}
-\frac{\log\lvert\mathbb{D}_l\rvert}{2}\cdot\mathop{\sum}\limits_{i=1}^{\lvert V_{\mathbf{G}'}\rvert}{\widehat{\Pi_{i}}\cdot(\widehat{v_i}-1)}\\
&=\mathop{\max}\limits_{\Phi^{(4)}}{\log\mathop{\prod}\limits_{i=1}^{k^*}\mathop{\prod}\limits_{j=1}^{1}\mathop{\prod}\limits_{k=1}^{M+1}}{(\phi_{ijk}^{(4)})^{n_{ijk}}}
-\frac{\log\lvert\mathbb{D}_l\rvert}{2}\cdot k^*\cdot M\\
&+\mathop{\max}\limits_{\Phi^{(0)}}{\log\mathop{\prod}\limits_{v_i\in W_{\mathbf{G}'}}\mathop{\prod}\limits_{j=1}^{\widehat{\Pi_i}}\mathop{\prod}\limits_{k=1}^{8}}{(\phi_{ijk}^{(0)})^{n_{ijk}}}
-\frac{\log\lvert\mathbb{D}_l\rvert}{2}\cdot\mathop{\sum}\limits_{v_i\in W_{\mathbf{G}'}}{\widehat{\Pi_{i}}\cdot7}.
\end{split}
\end{equation}

Let
\begin{equation}
\nonumber
\scriptsize
\begin{split}
\lambda=&(\mathop{\max}\limits_{\Phi^{(1)}}{\log\mathop{\prod}\limits_{i=1}^{k^*}\mathop{\prod}\limits_{j=1}^{\ell^{i-1}}\mathop{\prod}\limits_{k=1}^{\ell}}{(\phi_{ijk}^{(1)})^{n_{ijk}}}
-\frac{\log\lvert\mathbb{D}_l\rvert}{2}\cdot{k^*\cdot\ell^{i-1}\cdot(\ell-1)})\\
&+(\mathop{\max}\limits_{\Phi^{(2)}}{\log\mathop{\prod}\limits_{i=1}^{k^*}\mathop{\prod}\limits_{j=1}^{\ell}\mathop{\prod}\limits_{k=1}^{M+1}}{(\phi_{ijk}^{(2)})^{n_{ijk}}}
-\frac{\log\lvert\mathbb{D}_l\rvert}{2}\cdot{k^*\cdot\ell\cdot M})\\
&+(\mathop{\max}\limits_{\Phi^{(3)}}{\log\mathop{\prod}\limits_{i=1}^{k^*(k^*-1)/2}\mathop{\prod}\limits_{j=1}^{1}\mathop{\prod}\limits_{k=1}^{\mathrm{c}_i}}{(\phi_{ijk}^{(3)})^{n_{ijk}}}
-\frac{\log\lvert\mathbb{D}_l\rvert}{2}\cdot{\mathop{\sum}\limits_{i=1}^{k^*(k^*-1)/2}(\mathrm{c}_i-1)})\\
&+(\mathop{\max}\limits_{\Phi^{(0)}}{\log\mathop{\prod}\limits_{v_i\in V_{\mathbf{r}}}\mathop{\prod}\limits_{j=1}^{1}\mathop{\prod}\limits_{k=1}^{8}}{(\phi_{ijk}^{(0)})^{n_{ijk}}}
-\frac{\log\lvert\mathbb{D}_l\rvert}{2}\cdot\frac{k^*(k^*-1)}{2}\cdot 7)\text{\quad$\rhd$ for $v_i\in V_{\mathbf{r}}$ and $\Pi_i\subset V_{\mathbf{c}}$}\\
&-(\mathop{\max}\limits_{\Phi^{(4)}}{\log\mathop{\prod}\limits_{i=1}^{k^*}\mathop{\prod}\limits_{j=1}^{1}\mathop{\prod}\limits_{k=1}^{M+1}}{(\phi_{ijk}^{(4)})^{n_{ijk}}}
-\frac{\log\lvert\mathbb{D}_l\rvert}{2}\cdot k^*\cdot M).
\end{split}
\end{equation}

Given a dataset $\mathbb{D}_l$, the parameter vectors $\Phi^{(1)}$, $\Phi^{(2)}$, $\Phi^{(3)}$, $\Phi^{(4)}$ can be estimated to maximize their respective logarithmic equations regardless of the network structures. Similarly, for the parameters $\phi_{ijk}^{(0)}\in\Phi^{(0)}$ such that $v_i\in V_{\mathbf{r}}$ and $\Pi_i\subset V_{\mathbf{c}}$, they can also be estimated invariably under arbitrary network structures. Therefore, $\lambda$ is constant. Then we have
\begin{equation}
\nonumber
\small
\begin{split}
&\text{BIC-Score}(\mathbf{G}:\mathbb{D}_l)=\text{BIC-Score}(\mathbf{G}':\mathbb{D}_l)+\lambda.
\end{split}
\end{equation}
\end{proof}

\section{Conditional distribution for Gibbs sampling}
\label{appendix:sampling}
Since the probability of $\mathbf{a}_{n}$ conditioned on $\mathbf{t}_{n}$ is Dirichlet-multinomial distribution, we can write out the integration formulae:
\begin{equation}
\nonumber
\footnotesize
\begin{split}
&\prod_{d}\prod_{n=1}^{\mid d\mid}P(\mathbf{a}_{n}\mid \mathbf{t}_{n};\beta)=\prod_{k=1}^{\ell}{P(\mathbf{a}\mid T_k;\beta)}=\prod_{k=1}^{\ell}{\int_{\theta_{k}}{P(\mathbf{a}\mid\theta_{k})P(\theta_{k};\beta)}d{\theta_{k}}}\\
&=\prod_{k=1}^{\ell}\frac{\Gamma(\sum_{i'=1}^{M}{\beta_{k,i'}})}{\Gamma(\sum_{i'=1}^{M}{(na_{k,i'}+\beta_{k,i'})})}\prod_{i'=1}^{M}\frac{\Gamma(na_{k,i'}+\beta_{k,i'})}{\Gamma(\beta_{k,i'})}
\end{split}
\end{equation}
The total number of tables cannot exceed the maximum network size $\ell$.
%Notice that $P(\mathbf{a}_{1};\beta)=P(\mathbf{a}_{1}\mid\mathbf{t}_{1};\beta)$, where $\mathbf{t}_{1}$ is equivalent to $1$.

The probability of $\mathbf{t}_{n}$ conditioned on its previous tables $\mathbf{t}_{1},\mathbf{t}_{2},\ldots,\mathbf{t}_{n-1}$ follows CRP $(n\geq 2)$. We adopt the technique introduced by Sudderth~\cite{sudderth2006graphical} to sample the $n$-th table in $G_d$ as follows
\begin{equation}
\nonumber
\footnotesize
\begin{split}
&P(\mathbf{t}_{n}=T_{\zeta}\mid\mathbf{t}_{1},\mathbf{t}_{2},\ldots,\mathbf{t}_{n-1};\alpha_{\zeta})=\frac{1}{n+\alpha_{\zeta}-1}(\sum_{k=1}^{NT_{n}}{nt_{k}\delta(\zeta,k)+\alpha_{\zeta}\delta(\zeta,\bar{k})})
\end{split}
\end{equation}
where $nt_{k}$ is the number of time the previous $n-1$ nodes in $G_{d}$ are assigned to the table $T_k$, and $\delta$ is the Kronecker delta function that \[\footnotesize \delta(\zeta,k) = \left\{
  \begin{array}{l l}
    1 & \quad \text{if $\zeta=k$}\\
    0 & \quad \text{if $\zeta\neq k$}
  \end{array} \right.\]
and $\bar{k}$ denotes a previously empty table.

Now we derive the conditional probability of each variable $\mathbf{t}_{\mathbf{n}}$ of the $\mathbf{n}$-th table in $G_{\mathbf{d}}$ assigned to the table $T_{\zeta}$ by Gibbs sampling as follows:
%\begin{equation}
\begin{scriptsize}
\begin{align}
\nonumber
&P(\mathbf{t}_{\mathbf{n}}=T_{\zeta}\mid\mathbf{t}_{\mathbf{-n}}, \mathbf{a}, \mathbf{r};\alpha, \beta)\notag
\propto P(\mathbf{t}_{\mathbf{n}}, \mathbf{t}_{\mathbf{-n}}, \mathbf{a}, \mathbf{r};\alpha, \beta)\notag\\
&=\prod_{d}\prod_{n=1}^{\mid d\mid}P(\mathbf{a}_{n}\mid \mathbf{t}_{n};\beta)\notag
\times\prod_{d}\prod_{n=2}^{\mid d\mid}P(\mathbf{t}_{n}\mid\mathbf{t}_{1},\ldots,\mathbf{t}_{n-1};\alpha)\notag\times\prod_{d}\prod_{n=2}^{\mid d\mid}P(\mathbf{r}_{n-1,n}\mid \mathbf{a}_{n-1},\mathbf{a}_{n})\notag\\
&\propto\prod_{d}\prod_{n=1}^{\mid d\mid}P(\mathbf{a}_{n}\mid \mathbf{t}_{n};\beta)\times P(\mathbf{t}_{\mathbf{n}}\mid\mathbf{t}_{1},\ldots,\mathbf{t}_{\mathbf{n}-1};\alpha)\notag
\times \prod_{n\neq \mathbf{n}} P(\mathbf{t}_{n}\mid\mathbf{t}_{1},\ldots,\mathbf{t}_{n-1};\alpha)\notag\\
&\times \prod_{d\neq\mathbf{d}}\prod_{n=2}^{\mid d\mid} P(\mathbf{t}_{n}\mid\mathbf{t}_{1},\ldots,\mathbf{t}_{n-1};\alpha)\notag\\
&\propto\prod_{k=1}^{\ell}\frac{\Gamma(\sum_{i'=1}^{M}\beta_{k,i'})}{\Gamma(\sum_{i'=1}^{M}(na_{k,i'}+\beta_{k,i'} ))}\prod_{i=1}^{M}\frac{\Gamma(na_{k,i}+\beta_{k,i})}{\Gamma(\beta_{k,i})}\notag\times\frac{1}{\mathbf{n}+\alpha_{\zeta}-1}(\sum_{k=1}^{NT_{\mathbf{n}}}{nt_{k}\delta(\zeta,k)+\alpha_{\zeta}\delta(\zeta,\bar{k})})\notag\\
&\propto\frac{\prod_{i=1}^{M}{\Gamma(na_{\zeta,i}+\beta_{\zeta,i})}}{\Gamma(\sum_{i'=1}^{M}{(na_{\zeta,i'}}+\beta_{\zeta,i'}) )}\prod_{k\neq\zeta}\frac{\prod_{i=1}^{M}{\Gamma(na_{k,i}+\beta_{k,i})}}{\Gamma(\sum_{i'=1}^{M}{(na_{k,i'}}+\beta_{k,i'})
)}\notag\\
&\times\frac{1}{\mathbf{n}+\alpha_{\zeta}-1}(\sum_{k=1}^{NT_{\mathbf{n}}}{nt_{k}\delta(\zeta,k)+\alpha_{\zeta}\delta(\zeta,\bar{k})})\notag\\
&\propto\frac{\Gamma(na_{\zeta,\tilde{\mathbf{i}},-\mathbf{n}}+\beta_{\zeta,\tilde{\mathbf{i}}}+1)\prod_{i\neq\tilde{\mathbf{i}}}{\Gamma(na_{\zeta,i,-\mathbf{n}}+\beta_{\zeta,i})}}{\Gamma(\sum_{i'=1}^{M}({na_{\zeta,i',-\mathbf{n}}}+\beta_{\zeta,i'}) +1)}\notag
\prod_{k\neq\zeta}\frac{\prod_{i=1}^{M}{\Gamma(na_{k,i,-\mathbf{n}}+\beta_{k,i})}}{\Gamma(\sum_{i'=1}^{M}({na_{k,i',-\mathbf{n}}}+\beta_{k,i'}))}\notag\\
&\times\frac{1}{\mathbf{n}+\alpha_{\zeta}-1}(\sum_{k=1}^{NT_{\mathbf{n}}}{nt_{k}\delta(\zeta,k)+\alpha_{\zeta}\delta(\zeta,\bar{k})})\notag\\
&=\frac{na_{\zeta,\tilde{\mathbf{i}},-\mathbf{n}}+\beta_{\zeta,\tilde{\mathbf{i}}}}{\sum_{i'=1}^{M}{(na_{\zeta,i',-\mathbf{n}}+\beta_{\zeta,i'}) }}\prod_{k=1}^{\ell}\frac{\prod_{i=1}^{M}{\Gamma(na_{k,i,-\mathbf{n}}+\beta_{k,i})}}{\Gamma(\sum_{i'=1}^{M}({na_{k,i',-\mathbf{n}}}+\beta_{k,i'}))}\notag\\
&\times\frac{1}{\mathbf{n}+\alpha_{\zeta}-1}(\sum_{k=1}^{NT_{\mathbf{n}}}{nt_{k}\delta(\zeta,k)+\alpha_{\zeta}\delta(\zeta,\bar{k})})\notag\\
&\propto\frac{na_{\zeta,\tilde{\mathbf{i}},-\mathbf{n}}+\beta_{\zeta,\tilde{\mathbf{i}}}}{\sum_{i'=1}^{M}{(na_{\zeta,i',-\mathbf{n}}+\beta_{\zeta,i'})
}}\notag\times\frac{1}{\mathbf{n}+\alpha_{\zeta}-1}(\sum_{k=1}^{NT_{\mathbf{n}}}{nt_{k}\delta(\zeta,k)+\alpha_{\zeta}\delta(\zeta,\bar{k})})\notag\\
&= \left\{
  \begin{array}{l l}
  \nonumber
    \frac{na_{\zeta,\tilde{\mathbf{i}},-\mathbf{n}}+\beta_{\zeta,\tilde{\mathbf{i}}}}{\sum_{i'=1}^{M}{(na_{\zeta,i',-\mathbf{n}}+\beta_{\zeta,i'} )}}\times\frac{nt_{\zeta}}{\mathbf{n}+\alpha_{\zeta}-1} & \quad \text{if $\zeta\leq NT_{\mathbf{n}}$}\\
    \frac{na_{\zeta,\tilde{\mathbf{i}},-\mathbf{n}}+\beta_{\zeta,\tilde{\mathbf{i}}}}{\sum_{i'=1}^{M}{(na_{\zeta,i',-\mathbf{n}}+\beta_{\zeta,i'} )}}\times\frac{\alpha_{\zeta}}{\mathbf{n}+\alpha_{\zeta}-1} & \quad \text{if $\zeta=NT_{\mathbf{n}}+1$}
  \end{array} \right.
\end{align}
\end{scriptsize}
\end{appendices}

%\pagebreak
\bibliographystyle{spmpsci}
%\footnotesize
\bibliography{references}

\end{document}